\documentclass{article}

 \usepackage[main, final]{neurips2026/neurips_2026}

\usepackage[utf8]{inputenc} 
\usepackage[T1]{fontenc}    
\usepackage{hyperref}       
\usepackage{url}            
\usepackage{booktabs}       
\usepackage{amsfonts}       
\usepackage{nicefrac}       
\usepackage{microtype}      
\usepackage{xcolor}         
\usepackage{amssymb}
\usepackage{amsmath}
\usepackage{amsthm}
\newtheorem{theorem}{Theorem}[section]
\newtheorem{corollary}{Corollary}[theorem]
\newtheorem{lemma}[theorem]{Lemma}
\newtheorem{definition}[theorem]{Definition}
\newtheorem*{retheorem}{Theorem}
\newtheorem*{recorollary}{Corollary}
\newtheorem*{relemma}{Lemma}

\newcommand{\modelname}{UncertainGen}
\usepackage{wrapfig} 
\usepackage{graphicx}
\usepackage{enumitem} 

\usepackage{subcaption}
\usepackage{multirow}

\title{\modelname: Scalable Uncertainty-Aware Representation Learning of DNA Sequences}

\author{
  Abdulkadir \c{C}elikkanat\textsuperscript{1} \quad
  Andres R. Masegosa\textsuperscript{1} \quad
  Mads Albertsen\textsuperscript{2} \quad
  Thomas D. Nielsen\textsuperscript{1} \\
  Aalborg University \\
  9000 Aalborg, Denmark \\
  \textsuperscript{1}\texttt{\{abce,arma,tdn\}@cs.aau.dk}, \textsuperscript{2}\texttt{ma@bio.aau.dk}
}

\begin{document}

\maketitle

\begin{abstract}
Learning effective representations of DNA sequences is fundamental to genomic data analysis, especially in the presence of high sequence similarity and inter-species DNA sharing. Existing approaches rely on deterministic embeddings, such as k-mer statistics or representations from large language models, which fail to capture the inherent uncertainty in short and ambiguous DNA fragments. We propose \textsc{\modelname}, a lightweight probabilistic representation learning framework that embeds DNA sequences as distributions in a latent space. The learned embedding variance captures positional ambiguity, increasing for sequences compatible with multiple genomic clusters. We establish theoretical guarantees on embedding distinguishability, and show that the probabilistic formulation induces a data-adaptive metric that expands the effective representational space. We evaluate \textsc{\modelname} in the metagenomic binning task to demonstrate the practical benefits of uncertainty-aware representations. Experiments show that probabilistic embeddings consistently outperform deterministic k-mer and large genome foundation models, while remaining lightweight and efficient.
\end{abstract}
\section{Introduction\label{sec:introduction}}

Genomic sequences encode the blueprint of life, and analyzing them is fundamental for understanding biological processes, evolutionary relationships, and microbial ecosystems \citep{falkowskiMicrobialEnginesThat2008,timmisContributionMicrobialBiotechnology2017, cavicchioliScientistsWarningHumanity2019}. In recent years, advances in high-throughput DNA sequencing technologies have enabled large-scale studies of complex microbial communities directly from environmental samples. However, these technologies typically produce fragmented DNA sequences (called \emph{reads}) rather than complete genomes. This fragmentation poses a significant challenge: recovering the full DNA sequences of the microbes in a sample requires assembling these reads and organizing them according to their origin.

Organizing reads from a mixed microbial sample is known as \emph{metagenomic binning}, which aims to group DNA fragments so that each group corresponds to a distinct genome \citep{kuninBioinformaticiansGuideMetagenomics2008}. Accurate binning is critical for downstream analyses, such as functional annotation, phylogenetic profiling, and strain-level variation studies \citep{tempertonMetagenomicsMicrobialDiversity2012,meyerCriticalAssessmentMetagenome2022}. At its core, metagenomic binning relies on a representation of DNA fragments preserving genomic similarity and inter-species dissimilarity, enabling meaningful comparisons between reads or assembled contiguous sequences (i.e.\ contigs). This makes metagenomic binning a natural benchmark for evaluating DNA sequence representations, as it directly tests whether embeddings preserve genomic structure while offering well-defined ground truth.

Traditionally, DNA sequences are represented using \textit{$k$-mer profiles}, where sequences are decomposed into substrings of length $k$ to construct the feature vectors of DNA fragments (see Figure \ref{fig:motivating_example} (a-c)). Numerous methods leverage these $k$-mer–based representations to learn latent representations (i.e., \textit{embeddings}) and later cluster the DNA fragments for metagenomic binning  \citep{teelingApplicationTetranucleotideFrequencies2004,chanBinningSequencesUsing2008,pan2023semibin2,ccelikkanat2024revisiting,ji2021dnabert}. Recent studies employ Large Language Models (LLM) that operate directly on raw sequences by avoiding explicit $k$-mer feature vectors to generate embeddings with the aim of capturing richer contextual information \citep{nguyen2023hyenadna,zhou2023dnabert,zhou2024dnabert}. However, recent works also indicate that $k$-mer–based embeddings achieve comparable performance while offering orders-of-magnitude greater computational efficiency than large genome foundation models \citep{ccelikkanat2024revisiting}. 

\begin{figure*}[t]
\centering
\includegraphics[width=\textwidth]{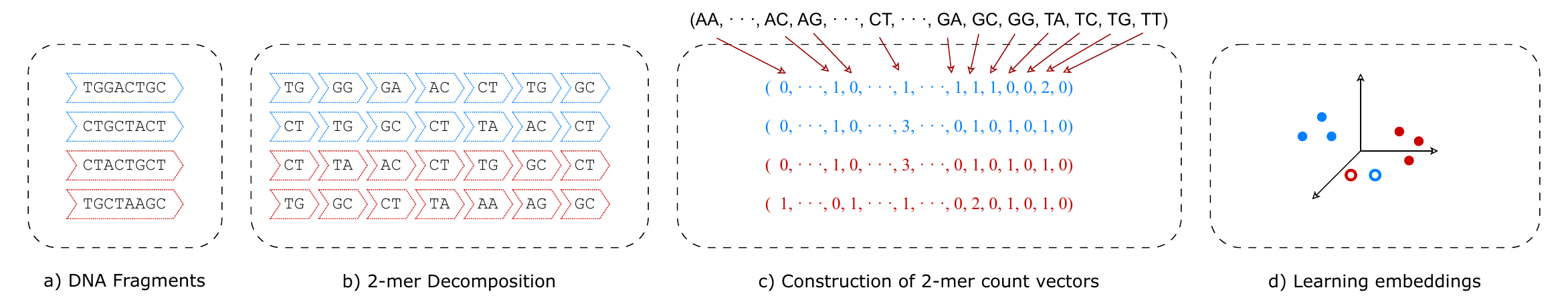}
\caption{\textbf{Illustration of the metagenomic binning process}. Starting from a set of DNA sequences (a), the process ends with their two-dimensional embeddings derived from 2-mer profiles (d). In general, these embeddings allow the DNA fragments from two different species to be correctly clustered. However, the second and third DNA sequences in (a) pose an exception: although distinct, their $k$-mer representations shown in (c) are highly similar and, consequently, their embeddings are also very close (shown as the two empty circled points in (d)). Because the k-mer profiles of DNA within a species tend to be (locally) similar, the contrastive learning procedure attempts to position such fragments in both clusters, but, since this is not possible, ultimately places them between them.}
\label{fig:motivating_example}
\end{figure*}

\begin{wrapfigure}{R}{0.4\textwidth}
\vspace{-1cm}
\begin{center}
\includegraphics[width=\linewidth]{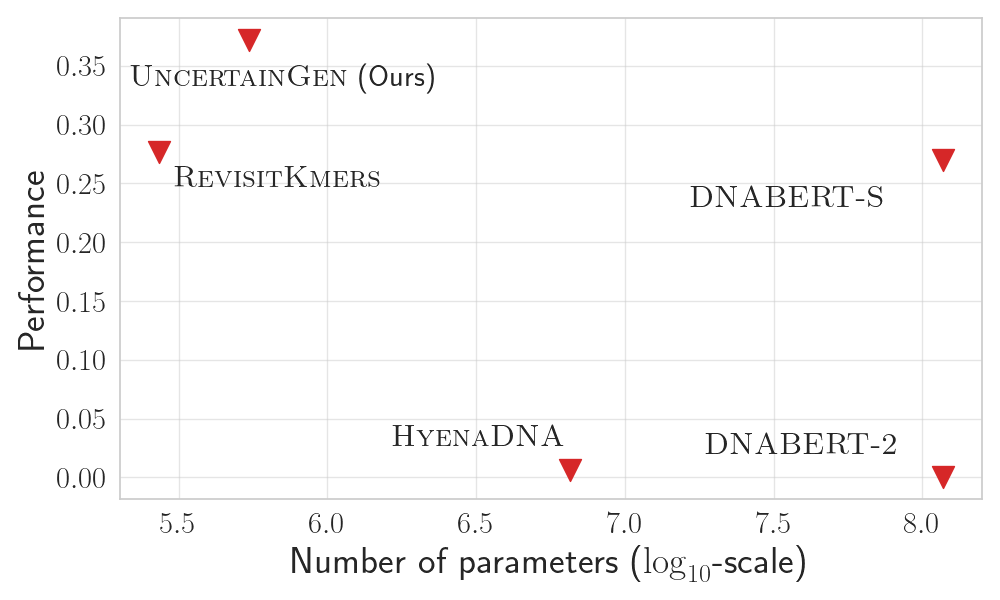}
\end{center}
\caption{\textbf{Performance versus model size}. \textsc{\modelname} (our approach) achieves a favorable trade-off between model size (log-scale) and the performance in metagenomic binning tasks compared with existing lightweight and large foundation models.}
\label{fig:performancevssize}
\end{wrapfigure}

A shared characteristic of these methods is that they produce \textit{deterministic embeddings}, mapping each DNA fragment to a single fixed point in the embedding space that is subsequently clustered and assigned to a single group representing the reconstructed species’ DNA (Figure \ref{fig:motivating_example} illustrates this process). However, many DNA sequences can appear in multiple genomes, for instance, through horizontal (lateral) gene transfer \citep{arnoldHorizontalGeneTransfer2022} and should ideally be assigned to their correct clusters. But this is impossible for any clustering algorithm because such sequences will be represented by the same point in the embedding space. A further limitation arises with $k$-mer–based representations: distinct sequences, potentially belonging to different clusters, can exhibit highly similar $k$-mer profiles and are therefore projected to similar embedding vectors, making it very difficult for the clustering algorithm to assign them accurately. Figure \ref{fig:motivating_example} visually illustrates this point.

In this work, we provide a mathematical formalization of the above issues and show that \textit{deterministic embeddings} cannot resolve them. We therefore propose the use of \textit{probabilistic embeddings} \citep{shiProbabilisticFaceEmbeddings2019a,warburgBayesianMetricLearning2023,karpukhinProbabilisticEmbeddingsRevisited2024}, where each fragment is mapped not to a single point but to a distribution (i.e., a region) in the embedding space. These distributions encode two sources of ambiguity: i) fragments may belong to multiple clusters (e.g., appear in the genomes of different species), and ii) identical $k$-mer profiles may arise from different species due to indistinguishable features.

Previous work on probabilistic embeddings in computer vision, NLP, and graph representation learning have typically relied on heuristically selected distributional distances, such as the Kullback–Leibler or Wasserstein divergence, to compare objects \citep{muzellec2018generalizing}. In contrast, our framework employs a contrastive-learning formulation in which non-Euclidean distances between embeddings emerge naturally from a probability distribution defined over the embedding space. This formulation yields closed-form expressions for the expected pairwise likelihood, enabling efficient and scalable optimization. Moreover, we present a theoretical analysis identifying the types of sequences producing large covariance terms, thereby offering insight into how and why the model captures uncertainty arising from ambiguous or multi-class sequences.

To the best of our knowledge, we propose the first framework for probabilistic embeddings of DNA sequences that extends the $k$-mer-based representation approaches  \citep{ccelikkanat2024revisiting,pan2023semibin2} to explicitly model ambiguity arising from  DNA fragments shared across multiple species. We propose a lightweight framework for embedding DNA sequences, \textsc{\modelname}, that consistently outperforms existing state-of-the-art approaches, spanning both lightweight methods and large foundation models as shown in Figure~\ref{fig:performancevssize}, offering both theoretical insights and practical performance gains. \textcolor{black}{Our focus is on learning uncertainty-aware representations rather than developing a complete binning pipeline; we evaluate embedding quality through metagenomic binning benchmarks, leaving integration with coverage information and other pipeline-specific components to future work.}

\begin{itemize}[left=5pt, labelsep=0.5em, itemsep=5pt, parsep=0pt, topsep=0pt]
\item We introduce a novel light-weight probabilistic embedding approach that models the uncertainty inherent in DNA sequences arising from inter-species DNA sharing and from fragments with highly similar representations.
\item We derive theoretical guarantees on the distinguishability of both deterministic and probabilistic embeddings, showing how the probabilistic embeddings expand the feasible latent space and improve the model’s capacity to separate DNA fragments.
\item We empirically demonstrate the effectiveness of our approach on six metagenomic binning datasets, using clustering performance as a proxy to evaluate representation quality. Our method shows consistent improvements over deterministic k-mer and LLM-based embeddings.
\end{itemize}

The implementation of the proposed approach and further details are provided in the \href{https://github.com/abdcelikkanat/uncertaingen}{project repository}.


\section{Related Work\label{sec:related_works}}

\textbf{Embedding models for DNA sequences.} Representations of DNA sequences have advanced rapidly in recent years. Classical approaches are based on $k$-mer profiles (frequency vectors of length-$k$ substrings) \citep{wu2016maxbin,kang2019metabat}, and many practical genome analysis tools continue to rely on such features \citep{nissen2021improved,wang2024effective,kutuzova2024binning}. More recently, the field has seen a surge of genome foundation models that adapt transformers or other long-context architectures to genomic data. \textsc{DNABERT} \citep{ji2021dnabert} introduced a BERT-style encoder with $k$-mer tokenization, while \textsc{DNABERT-2} \citep{zhou2023dnabert} replaced fixed $k$-mers with byte-pair encoding (BPE) to improve efficiency and downstream performance. \textsc{DNABERT-S} \citep{zhou2024dnabert} further refined this line of work by introducing a curriculum contrastive learning strategy with manifold instance mixup loss to address the metagenomic binning task. \textsc{HyenaDNA} \citep{nguyen2023hyenadna} extended the context window further by modeling single-nucleotide tokens with a long-range convolutional architecture, reducing the cost of dense attention. Other approaches explore alternative geometries, such as \textsc{HCNN} \cite{khan2025hyperbolic}, which learns sequence representations in hyperbolic space.

In parallel, lightweight but task-specific models have been developed for metagenomics. \textsc{SemiBin2} \cite{pan2023semibin2} and related methods \cite{wang2024effective} employ self-supervision and contrastive objectives tailored to binning, producing embeddings that cluster effectively by genome of origin. Recent work \citep{ccelikkanat2024revisiting} has revisited the foundations of $k$-mer features, showing both their scalability and the limits of when $k$-mer profiles alone can separate genomes in practice. These results highlight a central trade-off: while foundation models offer expressive, context-aware representations, lightweight contrastive or $k$-mer-based approaches can rival or even outperform them in realistic binning scenarios.

\textbf{Probabilistic embeddings for contrastive learning.} Probabilistic embeddings have previously been explored, both generally \citep{warburgBayesianMetricLearning2023,karpukhinProbabilisticEmbeddingsRevisited2024, bansalUnderstandingSelfSupervisedLearning2025} and within a task specific context \citep{vilnisWordRepresentationsGaussian2015,shiProbabilisticFaceEmbeddings2019a}. For example, in the context of face embeddings, each image is represented as a multivariate Gaussian distribution in embedding space, where a mutual likelihood score (MLS) is used to capture the likelihood of pairs of images belonging to the same person \cite{shiProbabilisticFaceEmbeddings2019a}. Property 1 in that work indicates that the proposed MLS score with fixed variance terms corresponds to a scaled and shifted negative squared Euclidean distance. Our results (Corollary~\ref{corollary:limitations}) extend Property 1 given in this paper \cite{shiProbabilisticFaceEmbeddings2019a} by characterizing the (limited) expressivity of the equivalent of a fixed variance MLS score, while at the same time showing that modeling capacity can be increased by allowing for varying covariance terms.            

Probabilistic embeddings have also been explored in a variational context. For instance,  the work \cite{ohModelingUncertaintyHedged2019} learns probabilistic embeddings using a soft contrastive loss while relying on a variational information bottleneck principle \citep{alemiDeepVariationalInformation2017}. The paper \cite{jeongProbabilisticVariationalContrastive2025} reinterprets the InfoNCE loss as a reconstruction term in the ELBO objective through an approximation of the decoder model, which effectively also makes the representation decoder free. The study \cite{kirchhofProbabilisticContrastiveLearning2023} posits a contrastive generative process and extends the InfoNCE loss to learn the correct posterior embedding distribution (up to rotation) for an unbounded number of negative samples; the correctness result relies on a known concentration parameter of the generative process for the positive samples. In contrast, we provide expressivity results related to model capacity, independent of any model-specific parameters defining the data-generating process.

\section{Proposed Model\label{sec:model}}

Let $\mathcal{S} \subset \Sigma^L$ be the set of sequences of length $L < \infty$ over alphabet $\Sigma := \{A,C,G,T\}$. In many genomic sequence clustering tasks, sequences originate from unknown genomes with only sparse pairwise similarity information; furthermore, the actual number of classes is unknown. In this regard, our goal is to learn an embedding function $\phi$ that captures the underlying cluster structure (i.e., \emph{binning}) while also modeling the uncertainty in embedding space. Each cluster is expected to contain DNA fragments belonging to the same species.

\textbf{Objective:} We aim to learn an embedding function, $\phi$, that maps sequences into a latent space, where the distance reflects cluster membership. Specifically, for a given threshold $\tau > 0$, we require:
\begin{align}
\| \phi(\mathbf{s}) - \phi(\mathbf{r}) \| < \tau \ \ \text{if and only if} \ \ \ell(\mathbf{s}) = \ell(\mathbf{r}) = k,
\end{align}
for some $k\in [K]$ where $\ell(\mathbf{s})$ denotes the cluster label of the DNA sequence $\mathbf{s}\in \mathcal{S}$. Therefore, an embedding function satisfying this condition maps sequences from the same cluster close together, and sequences from different clusters remain well separated.

\textbf{Lightweight Metagenomic Binning:} Our work builds on the works \cite{ccelikkanat2024revisiting,pan2023semibin2}, which introduced a state-of-the-art metagenomic binning algorithm. These methods achieve competitive accuracy while being several orders of magnitude faster than large genomic foundation models because they construct embeddings from efficient $k$-mer representations rather than using heavy sequence transformers. We adopt this principle of lightweight non-linear embeddings as the starting point for our approach.

In the works \cite{ccelikkanat2024revisiting,pan2023semibin2}, each DNA sequence in the dataset is split into two equal-length segments to form a \emph{positive} pair, while \emph{negative} pairs are designed by choosing two distinct sequences at random. For every segment pair, both $k$-mer profiles are computed and passed through a shared neural network to map them into an embedding space. The contrastive loss used there encourages embeddings of positive pairs to be close and embeddings of negative pairs to be far apart, thereby learning a genome-aware representation without supervision. These embeddings are later clustered with a standard algorithm, and all DNA fragments in a cluster are assumed to belong to a single species.  

Formally, let $\mathcal{S}=\{\mathbf{s}_i\}_{i=1}^N$ denote the set of DNA sequences in our dataset, with $\mathbf{s}_i^{(l)}$ and $\mathbf{s}_i^{(r)}$ being the left and right halves of each sequence. We construct triplets $\bigl\{(\mathbf{s}_i^{(l)},\mathbf{s}_j^{(r)},y_{ij})\bigr\}_{(i,j)\in\mathcal{I}}$ where $\mathcal{I}$ is the set of sequence index pairs, and $y_{ij}=1$ if the two segments originate from the same sequence (positive pair) and otherwise $y_{ij}=0$ (negative pair). The neural network parameters $\Omega$ are trained by minimizing
\begin{align}\label{eq:contrastive_loss}
\mathcal{L}\!\bigl(\Omega\bigr) =
-\!\!\!\sum_{(i,j)\in\mathcal{I}}
\!\Bigl[y_{ij}\log P(Y_{ij}=1 \mid \mathbf{s}^{(l)}_i, \mathbf{s}^{(r)}_j) +(1-y_{ij})\log\bigl(1-P(Y_{ij}=1 \mid \mathbf{s}^{(l)}_i, \mathbf{s}^{(r)}_j)\bigr)\Bigr],
\end{align}
with a success probability defined by $P(Y_{ij}=1 \mid \mathbf{s}_i^{(l)}, \mathbf{s}_j^{(r)}) := \exp\bigl(-\bigl\| \mathbf{z}_i^{(l)}-\mathbf{z}_j^{(r)}\bigr\|_2^2\bigr)$ where $\mathbf{z}_i$ denotes the embedding of $\mathbf{s}_i$. Since we assume that our dataset contains many different genomes (i.e., classes), negative pairs constructed by sampling random segments are more likely to originate from different genomes. Similarly, the positive pairs consist of fragments from the same genome due to the nature of the construction procedure. Hence, the model learns to produce similar embeddings for $k$-mer profiles from the same genome and dissimilar embeddings for profiles from different genomes.

\textbf{Motivation:} While the above contrastive framework provides an efficient way to learn genome-aware embeddings, its reliance on squared Euclidean distances between deterministic points (non-uncertain representations) in latent space imposes a fundamental limitation. As discussed in the introduction, many DNA fragments do not belong exclusively to a single cluster: they may genuinely occur in multiple genomes (e.g., through horizontal gene transfer \citep{arnoldHorizontalGeneTransfer2022}) or, conversely, fragments from distinct genomes may yield indistinguishable $k$-mer features. In both cases, the fixed-point embeddings produced by the model above collapse such sequences to the same location in the latent space, preventing any clustering algorithm from assigning them consistently. Lemma \ref{lemma:bounds} and Corollary \ref{corollary:limitations} (below) formalize this limitation by showing that it is not always possible to satisfy all pairwise constraints under deterministic embeddings.

In particular, if the set of DNA sequences is sufficiently large, then the DNA fragments originating from different clusters cannot simultaneously be mapped close to these cluster centroids to represent their membership while also being placed far apart from one another to reflect their pairwise dissimilarities. This motivates our shift to probabilistic embeddings, where each fragment is represented by a distribution, explicitly encoding ambiguity and thereby expanding the embedding space’s flexibility to better separate multi-cluster or otherwise indistinguishable sequences.

\textbf{Proposed model:} Our approach represents each DNA fragment as a Gaussian distribution, $\mathbf{z} \sim \mathcal{N}(\boldsymbol{\mu}, \mathbf{S})$ where two encoder networks output the mean--covariance pair $(\boldsymbol{\mu},\mathbf{S})$. This provides a principled way of modelling sequence-level ambiguity rather than evaluating the success probability $p(Y_{ij}=1\mid \mathbf{s}_i, \mathbf{s}_j)$ at fixed points. We define a new conditional distribution that \emph{marginalizes over the uncertainty of both embeddings} $(\mathbf{z}_i, \mathbf{z}_j)$,
\begin{align}\label{eq:bernoulli_success_prob_general_form}
P(Y_{ij}=1 \mid \mathbf{s}_i, \mathbf{s}_j) 
&=\mathbb{E}_{\substack{\mathbf{z}_{i} \sim \mathcal{N}(\boldsymbol{\mu}_{i}, \mathbf{S}_{i}) \\ \mathbf{z}_{j} \sim \mathcal{N}(\boldsymbol{\mu}_{j}, \mathbf{S}_{j})}}\left[ p(Y_{ij}=1 \mid \mathbf{z}_i, \mathbf{z}_j )\right]\nonumber
\\
&= \mathbb{E}_{\substack{\mathbf{z}_{i} \sim \mathcal{N}(\boldsymbol{\mu}_{i}, \mathbf{S}_{i}) \\ \mathbf{z}_{j} \sim \mathcal{N}(\boldsymbol{\mu}_{j}, \mathbf{S}_{j})}}\left[ \exp\!\left(-\frac{1}{2}(\mathbf{z}_i - \mathbf{z}_j)^\top \mathbf{K}_{ij}^{-1} (\mathbf{z}_i - \mathbf{z}_j) + \textcolor{black}{\beta} \right)\right],
\end{align}

where $\mathbf{K}_{ij} \succ 0$ is a positive definite matrix and $\beta$ accounts for the volume of the latent Gaussian distributions (discussed below). The inner part of the expectation captures the (unnormalized) Gaussian likelihood of the embedding difference $\mathbf{z}_i-\mathbf{z}_j\sim \mathcal{N}(\mathbf{0}, \mathbf{K}_{ij})$, and $\mathbf{K}_{ij}$ can therefore be seen as representing the uncertainty in the distance between $\mathbf{z}_i$ and $\mathbf{z}_j$ providing different weights to the differences across each embedding dimension; $\mathbf{K}_{ij}$ can thus also be interpreted as a local metric tensor.  By Lemma \ref{lemma:closed_form_expectation}, we can find the closed-form solution of the expectation term over embeddings (proofs of all formal results are placed in Appendix \ref{app:proofs}):

\begin{lemma}(Closed-form expectation)\label{lemma:closed_form_expectation}
Let $\mathbf{z}_i \sim \mathcal{N}\left( \boldsymbol{\mu}_i, \mathbf{S}_i \right)$ and $\mathbf{z}_j \sim \mathcal{N}\left( \boldsymbol{\mu}_j, \mathbf{S}_j \right)$ be independent random variables. For a given positive definite matrix $\mathbf{K}_{ij} \succ 0$, the expectation in Eq. \ref{eq:bernoulli_success_prob_general_form} is equal to
\begin{align}
\mathbb{E}_{\substack{\mathbf{z}_{i} \sim \mathcal{N}(\boldsymbol{\mu}_{i}, \mathbf{S}_{i}) \\ \mathbf{z}_{j} \sim \mathcal{N}(\boldsymbol{\mu}_{j}, \mathbf{S}_{j})}}\left[ p(Y_{ij}=1 \mid \mathbf{z}_i, \mathbf{z}_j )\right] =\! \frac{ 1 }{\sqrt{ |\mathbf{K}_{ij}^{-1}\mathbf{S}_{ij} + \mathbf{I}| }}  \exp\left( \!-\frac{1}{2}\boldsymbol{\mu}_{ij}^{\top} \left(  \mathbf{S}_{ij} \!+\! \mathbf{K}_{ij}\right)^{-1}\!\!\boldsymbol{\mu}_{ij} +\! \beta \right),
\end{align}
where $\boldsymbol{\mu}_{ij} := \boldsymbol{\mu}_{i} - \boldsymbol{\mu}_j$ and $\mathbf{S}_{ij} := \mathbf{S}_i + \mathbf{S}_j$.
\end{lemma}
As $\mathbf{K}_{ij}$ (the local metric tensor) can be seen as representing the uncertainty in the distance between $\mathbf{z}_i$ and $\mathbf{z}_j$, we let $\mathbf{K}_{ij} := \alpha(\mathbf{S}_i + \mathbf{S}_j)$ so that $\mathbf{K}_{ij}$ reflects the point-wise uncertainty in the embeddings.
The parameter $\alpha$ can be learnable, but in the remainder of the paper we set $\alpha=1$; Appendix~\ref{sec:alpha} includes further insight into the role of $\alpha \in \mathbb{R}^+$. With this choice, the expectation in Eq. \ref{eq:bernoulli_success_prob_general_form} simplifies algebraically to \textcolor{black}{ $2^{-D/2}\exp\left( \!-\frac{1}{4} (\boldsymbol{\mu}_{i} - \boldsymbol{\mu}_j)^{\top} (\mathbf{S}_{i}+\mathbf{S}_{j})^{-1}(\boldsymbol{\mu}_{i} - \boldsymbol{\mu}_j) + \beta \right)$. In high-dimensional spaces, the overlap between two probability distributions shrinks rapidly as the dimensionality $D$ grows. This causes the expected similarity to tend toward zero, even for two identical sequences. By setting $\beta:=D\ln{2} / 2$ we always have that $P(Y_{ii}=1 \mid \mathbf{s}_i, \mathbf{s}_i) = 1$, and Eq. \ref{eq:bernoulli_success_prob_general_form} simplifies to:
}
\begin{align}\label{eq:algebraic_simplification}
P(Y_{ij}=1 \mid \mathbf{s}_i, \mathbf{s}_j) =  \exp\!\left( \!-\frac{1}{4} (\boldsymbol{\mu}_{i} - \boldsymbol{\mu}_j)^{\top} (\mathbf{S}_{i}+\mathbf{S}_{j})^{-1}(\boldsymbol{\mu}_{i} - \boldsymbol{\mu}_j) \right).
\end{align}
Here, it is worth emphasizing that both the mean vectors $(\boldsymbol{\mu}_i, \boldsymbol{\mu}_j)$ and covariance matrices $(\mathbf{S}_i, \mathbf{S}_j)$ are parameterized by two simple neural networks denoted by $\phi_\mu$ and $\phi_\sigma$. They consist of a single hidden layer including $512$ units with sigmoid activation functions, and the output dimension (i.e., $D$) is set to $256$. We optimize these neural networks by minimizing the log-likelihood in Eq. \ref{eq:contrastive_loss} with the success probability in Eq. \ref{eq:algebraic_simplification}. Details about the optimization is provided in Section \ref{sec:experiments}.




\begin{definition}[$\epsilon$-distinguishable]\label{def:distinguishable}
For a given $\epsilon \in (0, 1/2 )$, a probabilistic mapping $\phi=(\phi_{\mu}, \phi_{\sigma}): \mathcal{S} \to \mathbb{R}^D\times \mathbb{R}^D_{+}$ parametrized by the mean and variance functions $\phi_{\mu}: \mathcal{S} \to \mathbb{R}^D$ and $ \phi_{\sigma}: \mathcal{S} \to \mathbb{R}^D_{+}$ is an \textbf{$\boldsymbol{\epsilon}$-distinguishable embedding} function for a set $\{ (\mathbf{s}_i, \mathbf{s}_j, y_{ij}) \}_{(i,j) \in \mathcal{I}} \subseteq \mathcal{S}^2 \times \{0,1\}$ if it satisfies:
\begin{align}
P(Y_{ij}=y_{ij} \mid \mathbf{s}_i, \mathbf{s}_j) \geq (1-\epsilon)
\end{align}

for all triplets $(\mathbf{s}_i, \mathbf{s}_j, y_{ij})$ indexed by $\mathcal{I}$. We omit $y_{ij}$ for brevity when it is clear from context.
\end{definition}

Intuitively, an $\epsilon$-distinguishable embedding function guarantees that positive pairs remain close in the latent space, while negative pairs are sufficiently separated. Lemma~\ref{lemma:bounds} formalizes this relationship by providing explicit bounds on the Euclidean distance between the embedding means as a function of the corresponding variances, thus offering theoretical guarantees for pairwise distinguishability.

\begin{lemma}\label{lemma:bounds}
Let $\epsilon \in (0, 1/2 )$, and let $\phi = (\phi_{\mu}, \phi_{\sigma}): \mathcal{S} \to  \mathbb{R}^D \times \mathbb{R}^D_{+}$ be an $\epsilon$-distinguishable embedding function for a triplet $(\mathbf{s}_i,\mathbf{s}_j,y_{ij})\in \mathcal{S}^2 \times \{0,1\}$. Then the following bounds hold:
\begin{align}
&\min_{d}\left\{ (\phi_\sigma(\mathbf{s}_i) + \phi_\sigma(\mathbf{s}_j))_d \right\}
 \log\left( \frac{1}{ \epsilon^4 }  \right) \leq \|\phi_\mu(\mathbf{s}_i) - \phi_\mu(\mathbf{s}_j)\|^2_2 &&\!\!\!\!\!\!\!\!\!\! \text{if $y_{ij}=0$,}
\\
&\max_{d}\left\{ (\phi_\sigma(\mathbf{s}_i) + \phi_\sigma(\mathbf{s}_j))_d \right\}
 \log\left( \frac{1}{(1-\epsilon)^{4}} \right) \geq \|\phi_\mu(\mathbf{s}_i) - \phi_\mu(\mathbf{s}_j)\|^2_2 &&\!\!\!\!\!\!\!\!\!\! \text{if $y_{ij}=1$.}
\end{align}
\end{lemma}

\begin{wrapfigure}{r}{0.35\textwidth}
\begin{center}
\includegraphics[width=\linewidth]{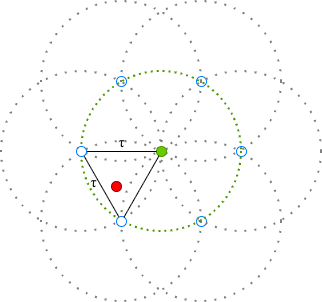}
 \end{center}
\caption{Packing number ($\mathcal{P}_\tau^D$).}
\label{fig:packing_number}
\end{wrapfigure}

These bounds highlight the critical role of variances in our setting. Intuitively, if a sequence is associated with multiple clusters, its mean representation $(i.e., \phi_\mu(\mathbf{s}))$ cannot be simultaneously close to all corresponding centroids, since these centroids must remain sufficiently separated. For instance, Figure~\ref{fig:packing_number} depicts six blue points placed at distance $\tau$ from a central green point, while also being pairwise $\tau$-separated. In this two-dimensional configuration, it is not possible to place an additional point (such as the red point) that lies within distance $\tau$ of the green point while remaining more than $\tau$ away from all the blue points. Hence, the covariance terms $\phi_\sigma(\mathbf{s})$ must adaptively increase to satisfy the distinguishability condition in Eq.~\ref{def:distinguishable}. 

Before stating the following corollary, we first introduce the \emph{packing number} $\mathcal{P}^D_\tau$, defined as the maximum number of $\tau$-separated distinct points in a ball of radius $\tau$ in $D$-dimensional space \citep{vershynin2018high}. It will help us to formalize the limitations of the embedding function.
\begin{corollary}\label{corollary:limitations}
Let $\phi = (\phi_{\mu}, \phi_{\sigma}): \mathcal{S} \to  \mathbb{R}^D \times \mathbb{R}^D_{+}$ be an embedding function with constant variance, such that $\phi_{\sigma}(\mathbf{s})_d = \sigma^2$ is invariant for all $\mathbf{s} \in \mathcal{S}$ and $d \in [D]$. If there exists a collection of $N \geq \mathcal{P}^D_\tau +1$ sequences $\{\mathbf{s}_1, \dots, \mathbf{s}_{N}\} \subseteq \mathcal{S}$ that are each positive pairs with a common sequence $\mathbf{s}_0 \in \mathcal{S}$, but are pairwise negative with respect to each other, then $\phi$ cannot be an $\epsilon$-distinguishable embedding function for any $\epsilon \in (0, 1/2)$.
\end{corollary}
%
This packing construction is an expressivity result and does not imply that the extremal configuration is attained in the benchmark datasets. From Corollary~\ref{corollary:limitations}, we see that fixed-variance embeddings have intrinsic limitations in expressiveness: they cannot simultaneously satisfy the pairwise constraints of a sufficiently large sequence set. Allowing covariance terms to vary introduces additional degrees of freedom, increasing the modeling capacity of the embedding function. Theorem~\ref{thm:theorem} further characterizes how these variances must adapt when increasingly tight distinguishability is required for a sequence that forms positive pairs with multiple mutually negative sequences.
\begin{theorem}\label{thm:theorem}
Let $\{\phi^t=(\phi_\mu,\phi_\sigma^t)\}_{t\geq1}$ have fixed bounded means and $\max_d\phi_\sigma^t(\mathbf{s}_i)_d \leq \kappa\min_d\phi_\sigma^t(\mathbf{s}_i)_d$ $\forall\, i\in[N],\,t\ge1$ for some $\kappa < \infty$. Consider $\{\mathbf{s}_i\}_{i=0}^N$, $N\geq2$, such that $(\mathbf{s}_0,\mathbf{s}_i)$ are positive and $(\mathbf{s}_i,\mathbf{s}_j)$ are negative for all $i\neq j\in[N]$, with pairwise distinct mean embeddings. Assume that, for every $d\in[D]$, there exists $i\in[N]$ such that $\phi_\mu(\mathbf{s}_0)_d\neq\phi_\mu(\mathbf{s}_i)_d$. Then $\phi^t$ is eventually $\epsilon$-distinguishable for every $\epsilon\in(0,1/2)$ if and only if
$\max_{i,d}\phi_\sigma^t(\mathbf{s}_i)_d\to0$ and $\min_d\phi_\sigma^t(\mathbf{s}_0)_d\to\infty$.
\end{theorem}
%
This result underscores the importance of probabilistic embeddings: by having covariance terms, the model can represent complex relationships among sequences that deterministic embeddings cannot capture. We note that we represent each DNA fragment as a single Gaussian distribution rather than a mixture model in order to preserve scalability and maintain the closed-form formulation in Eq. \ref{eq:algebraic_simplification}. Moreover, in metagenomic settings, the number of species a DNA fragment belongs to is unknown beforehand, making species-specific mixture representations more difficult to estimate and optimize.

\section{Experiments\label{sec:experiments}}

We evaluate \textsc{\modelname} using the same experimental setup as \citep{zhou2024dnabert,ccelikkanat2024revisiting} to ensure a fair comparison with previous deterministic and large genome foundation models. 
We also employ the same benchmark datasets, introduced by the work \cite{zhou2024dnabert} and adopted in subsequent study \cite{ccelikkanat2024revisiting}, which provide well-defined ground truth labels essential for evaluating embedding quality. 
Reference $5$/$6$ are synthetically constructed from fungal and viral reference genomes, whereas Plant $5$/$6$ and Marine $5$/$6$ are derived from the more complex CAMI2 settings; additional dataset statistics are provided in Appendix~\ref{appendix:experiments}.
Although real metagenomic samples would test the full performance of the pipeline, controlled benchmarks allow us to isolate the contribution of the representation method from confounding factors such as assembly quality, coverage variation, and chimeric sequences. This controlled setting is particularly appropriate for our goal of evaluating whether probabilistic embeddings better capture genomic structure than deterministic alternatives.  We provide further details about baselines and datasets in Appendix~\ref{appendix:experiments}. 

\textbf{Training procedure.} For our method, training was performed within a contrastive learning framework by optimizing the objective function given in Eq. \ref{eq:contrastive_loss}. Positive pairs were obtained by splitting each fragment into two halves, guaranteeing that both subsequences come from the same genome. Negative pairs were formed by randomly pairing fragments from the dataset; this procedure ensures, with high probability, that the paired sequences come from different genomes. From the original dataset consisting of two million pairs \citep{zhou2024dnabert}, we randomly subsample $10^5$ pairs to demonstrate that our method is effective even with smaller training sets compared to large genome foundation models. For each positive pair, we generated $200$ negative pairs, resulting in a total of $2,01 \times 10^7$ pairs.

We trained our model using the Adam optimizer with a learning rate of $10^{-2}$ and a batch size of $10^5$ pairs. To improve stability, we first trained the mean encoder network alone for $50$ epochs where we considered the squared Euclidean distance instead of Eq. \ref{eq:algebraic_simplification}, more specifically $\exp( -\| \boldsymbol{\mu}_i - \boldsymbol{\mu}_j \|_2^2 / 4 )$. Then, we trained the variance network only for an additional $25$ epochs with Eq. \ref{eq:algebraic_simplification}. 

We employed a thresholding strategy to bound the embedding space and avoid unnecessary updates from already separated negative pairs  ($y_{ij}=0$). Specifically, for negative pairs, the distance inside the exponential term was clipped at a fixed threshold, so that the resulting gradient is zero once the pair is sufficiently far apart. This follows the standard margin intuition in contrastive metric learning: negative pairs need only satisfy a separation constraint, rather than be pushed indefinitely far apart.

It should be emphasized that this thresholding strategy remains entirely consistent with the theoretical guarantees established in Section \ref{sec:model}. Our distinguishability framework (Definition \ref{def:distinguishable}) requires that for any pair, the success probability $P(Y_{ij} = y_{ij} \mid \mathbf{s}_i, s_j)$ meets or exceeds a boundary $(1 - \epsilon)$. By discarding the gradients once a negative pair is sufficiently separated, we simply recognize that the specific pairwise distinguishability constraint has been satisfied. Therefore, the optimization procedure prioritizes updates for pairs that do not yet meet the theoretical criteria for separation or proximity, ensuring that the final learned representations respect the distinguishability bounds. Further details and discussions can be found in Appendix \ref{appendix:experiments}.

\begin{figure}[t]
\centering
\begin{subfigure}{0.49\textwidth}
  \centering
  \includegraphics[width=\linewidth]{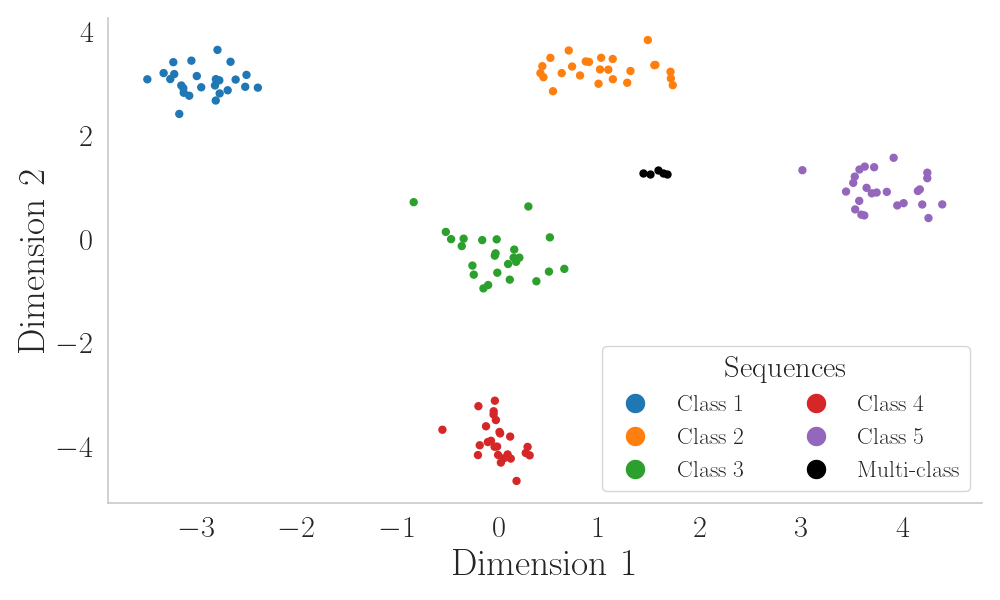}
  \caption{Learned embeddings}
  \label{fig:toy_example_embs}
\end{subfigure}%
\vspace{0.1em}
\begin{subfigure}{0.49\textwidth}
  \centering
  \includegraphics[width=\linewidth]{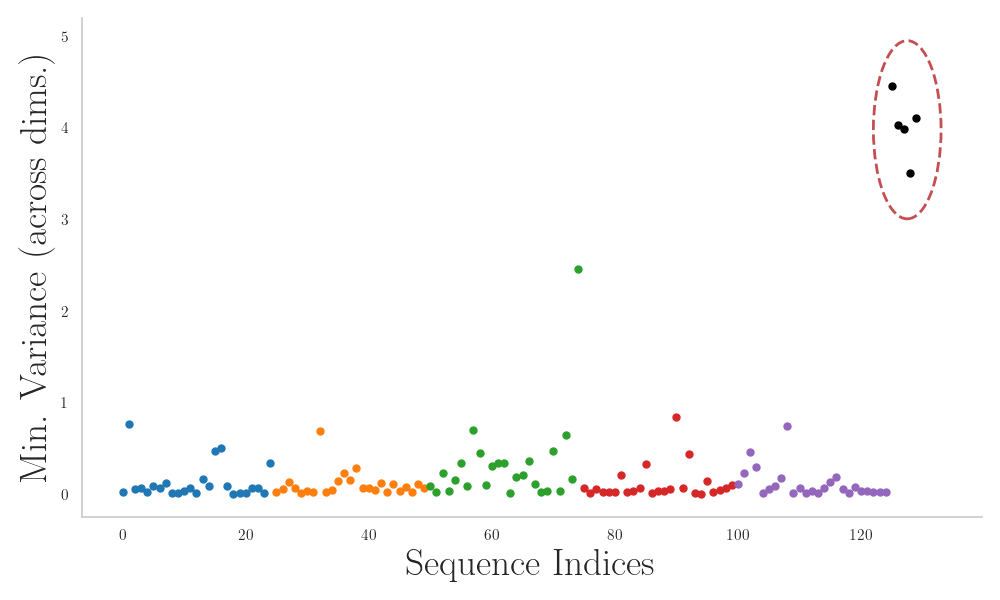}
  \caption{Variance distribution}
  \label{fig:toy_example_var_distr}
\end{subfigure}
\caption{Visualization of the learned embeddings and the variance distribution of the sequences.}
\label{fig:test}
\end{figure}

\begin{figure*}[t]
\centering
\includegraphics[width=0.99\textwidth]{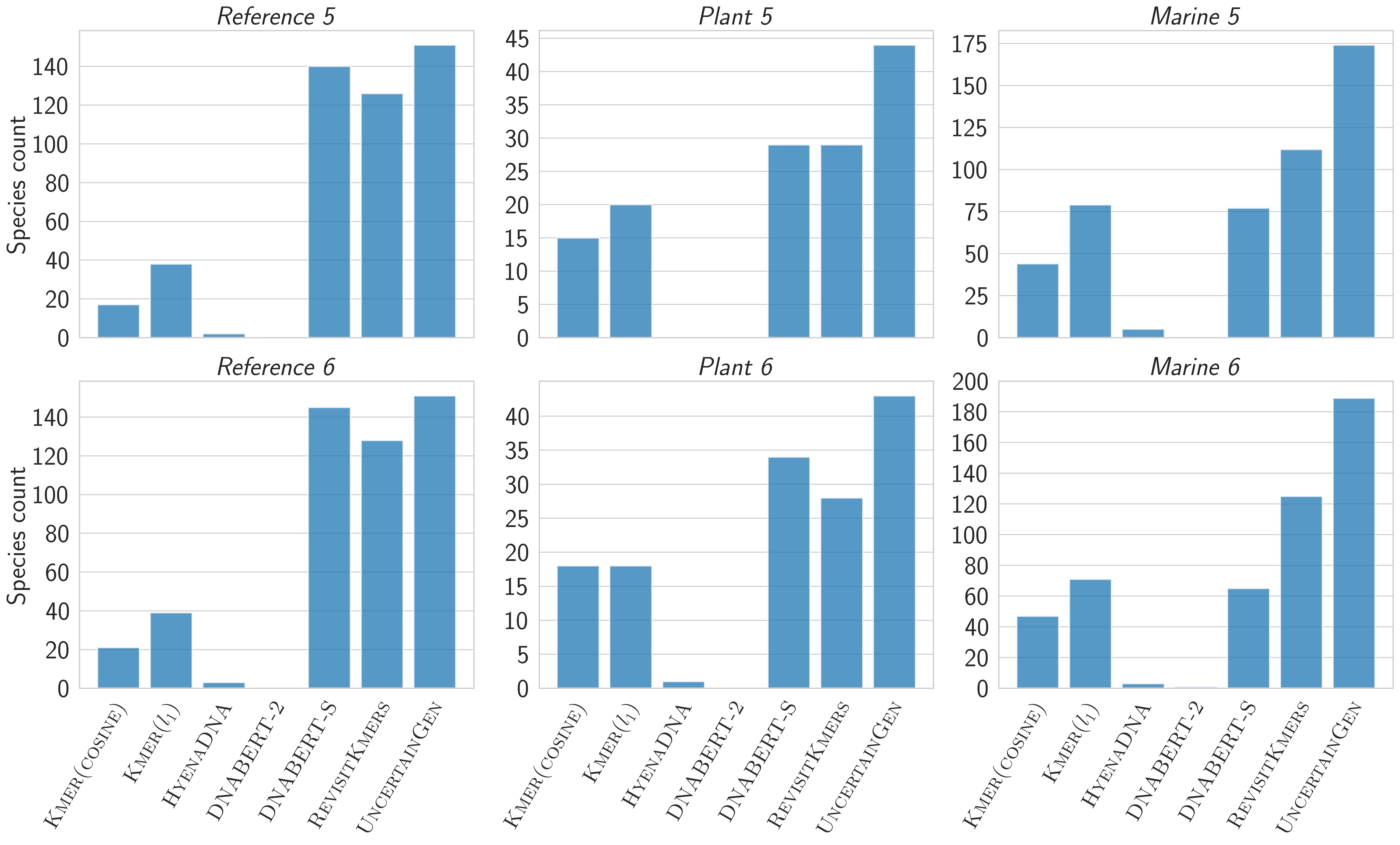}
\caption{\textbf{Metagenomic binning results.} The number of detected HQ clusters ($F_1 > 0.9$).}
\label{fig:Metagenomic_binning_results}
\end{figure*}

\subsection{Toy Example with $k$-mer Dataset}\label{sec:toy}
To investigate our models' behavior in a controlled setting, we designed a synthetic dataset of $4$-mer sequences. This setup allows us to assess whether the model can learn meaningful low-dimensional embeddings that reflect cluster structure and how it represents sequences spanning multiple classes.

We generated sequences of length $100$ from multinomial distributions defining $5$ distinct classes, each with a characteristic $4$-mer compositions. To simulate ambiguity, we additionally generated $5$ "multi-class" sequences by combining $k$-mer counts from multiple classes, representing inputs that do not belong exclusively to a single class. This design allows us to evaluate how the model handles both well-separated clusters and overlapping class memberships.

Positive sequence pairs were formed by sampling sequences from the same class, while negative pairs were drawn across different classes, which introduces the possibility of false negatives. For multi-class sequences, pairs were also constructed with sequences from their contributing classes, allowing the model to learn representations that account for both pure and mixed memberships. Due to space limitations, we provide the details in Appendix \ref{appendix:toy_example}.

We learn the embeddings in a $2$-dimensional space, enabling direct visualization of the learned geometry without relying on dimensionality reduction techniques that could distort structural relationships (Figure \ref{fig:toy_example_embs}). The resulting embeddings reveal well-separated clusters corresponding to distinct species, while sequences that belong to multiple classes occupy intermediate regions (black points).

For each sequence, our model predicts a diagonal covariance matrix quantifying the degree of uncertainty in its placement. This uncertainty is particularly pronounced for multi-class sequences, as reflected in their larger covariance values in Figure \ref{fig:toy_example_var_distr}. In line with our theoretical results (Lemma \ref{lemma:bounds} and Theorem~\ref{thm:theorem}), the minimum variance across dimensions (i.e. $\min_d\{ (\phi_\sigma(\mathbf{s}))_d \}$) provides a lower bound on pairwise distances, and sequences associated with multiple classes indeed display a higher minimum variance. It shows that the model not only separates clusters effectively, but also encodes the uncertainty of ambiguous cases. We further compare relative distances in the embedding space in Appendix~\ref{appendix:toy_example}, showing that accounting for uncertainty through the generalized Mahalanobis distance places ambiguous sequences substantially closer to their constituent classes than the Euclidean distance.

\subsection{Metagenomics Binning}
We evaluate our methods on the metagenomic binning task, where the objective is to cluster sequences into species-level groups without prior knowledge of the number of clusters. In this regard, we adopt the modified K-Medoid algorithm \cite{zhou2024dnabert}, which jointly estimates the cluster assignments and the underlying number of species. This setting is particularly challenging as it requires models to provide representations that are simultaneously discriminative and robust under unsupervised partitioning. For computing similarities between sequences, we employ cosine similarity for all genome-scale foundation models as well as the \textsc{KMers(cosine)} baseline. In contrast, we use an exponential kernel over the $\ell_1$ distance for \textsc{KMers($\ell_1$)}, and $\ell_2$ distance for \textsc{RevitKmers}. We use an exponential kernel over the generalized Mahalanobis term in Eq.\ref{eq:algebraic_simplification} as a natural choice for our model.

For the evaluation of the models, we consider the number of detected high-quality bins (i.e., the number of clusters whose $F_1$ score is greater than $0.9$), and a more detailed comparison is provided in the appendix. As shown in Figure \ref{fig:Metagenomic_binning_results}, across all datasets, \textsc{\modelname} consistently outperforms its deterministic counterparts. Moreover, our method also surpasses the strongest competing genome-scale foundation model, \textsc{DNABERT-S} across all datasets, while at the same time providing a light-weight alternative to the foundation models. The latter property is illustrated in Figure~\ref{fig:performancevssize}, showing the relation between model size and the average number of high-quality bin ratios. \textsc{\modelname} also exhibits high performance stability; detailed five-run statistics are provided in Appendix \ref{appendx:sec:metagenomics_binning}. These results demonstrate that our approach is not only competitive with state-of-the-art foundation models but also offers the added benefit of a principled probabilistic formulation, enabling more robust and interpretable clustering in metagenomic settings with a smaller number of parameters.

\begin{figure*}[t]
     \centering
     \begin{subfigure}[b]{0.32\textwidth}
         \centering
         \includegraphics[width=\textwidth]{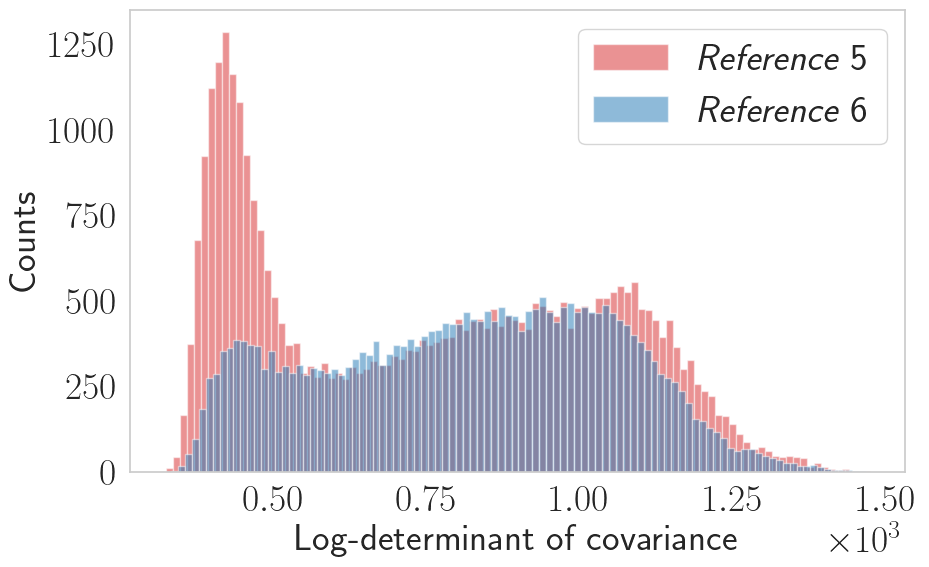}
         \caption{Log-determinant of variances.}
         \label{fig:log_determinant_distribution}
     \end{subfigure}
     \hfill
     \begin{subfigure}[b]{0.32\textwidth}
         \centering
         \includegraphics[width=\textwidth]{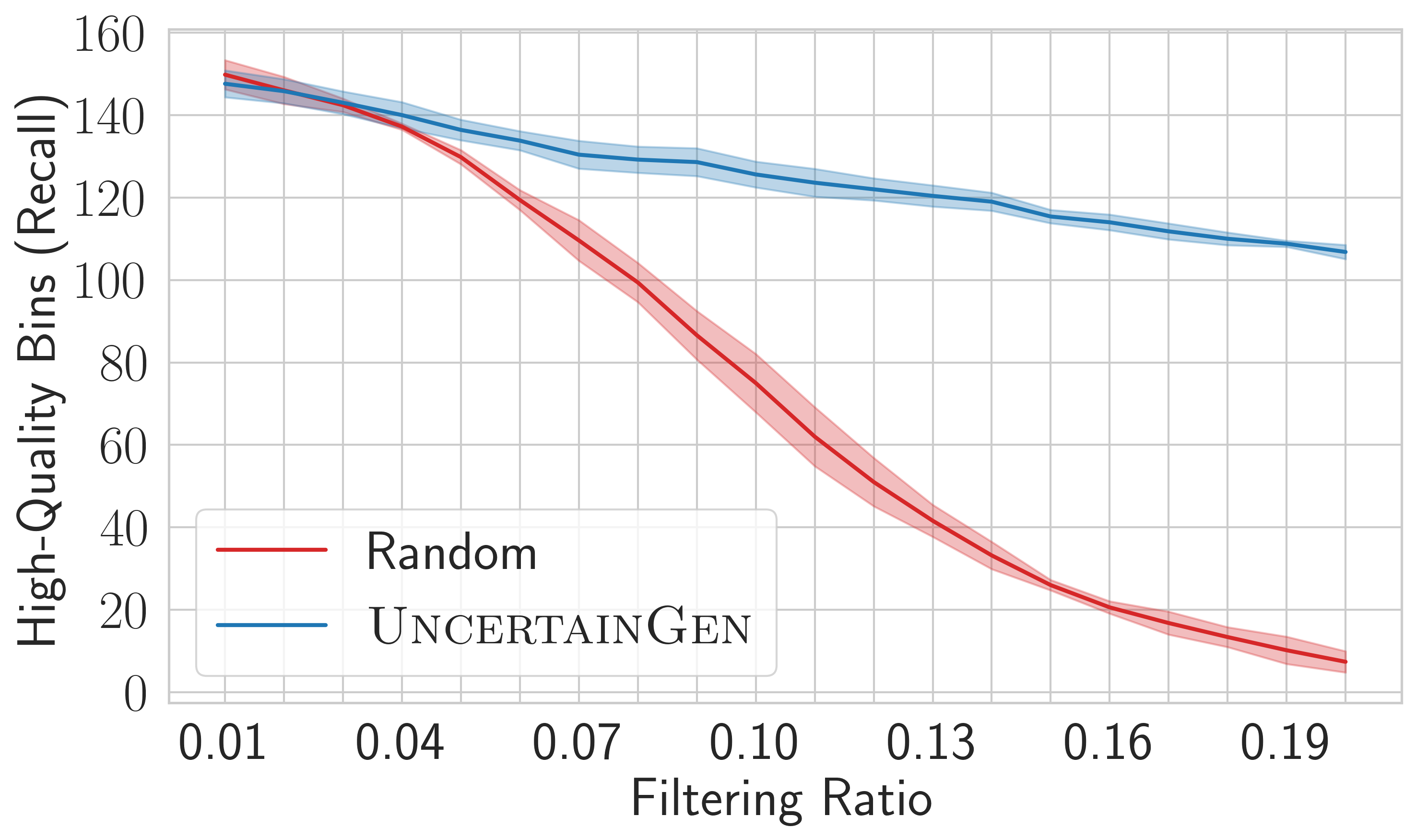}
         \caption{Sequence filtering}
         \label{fig:filtering_impact}
     \end{subfigure}
     \hfill
     \begin{subfigure}[b]{0.32\textwidth}
         \centering
         \includegraphics[width=\textwidth]{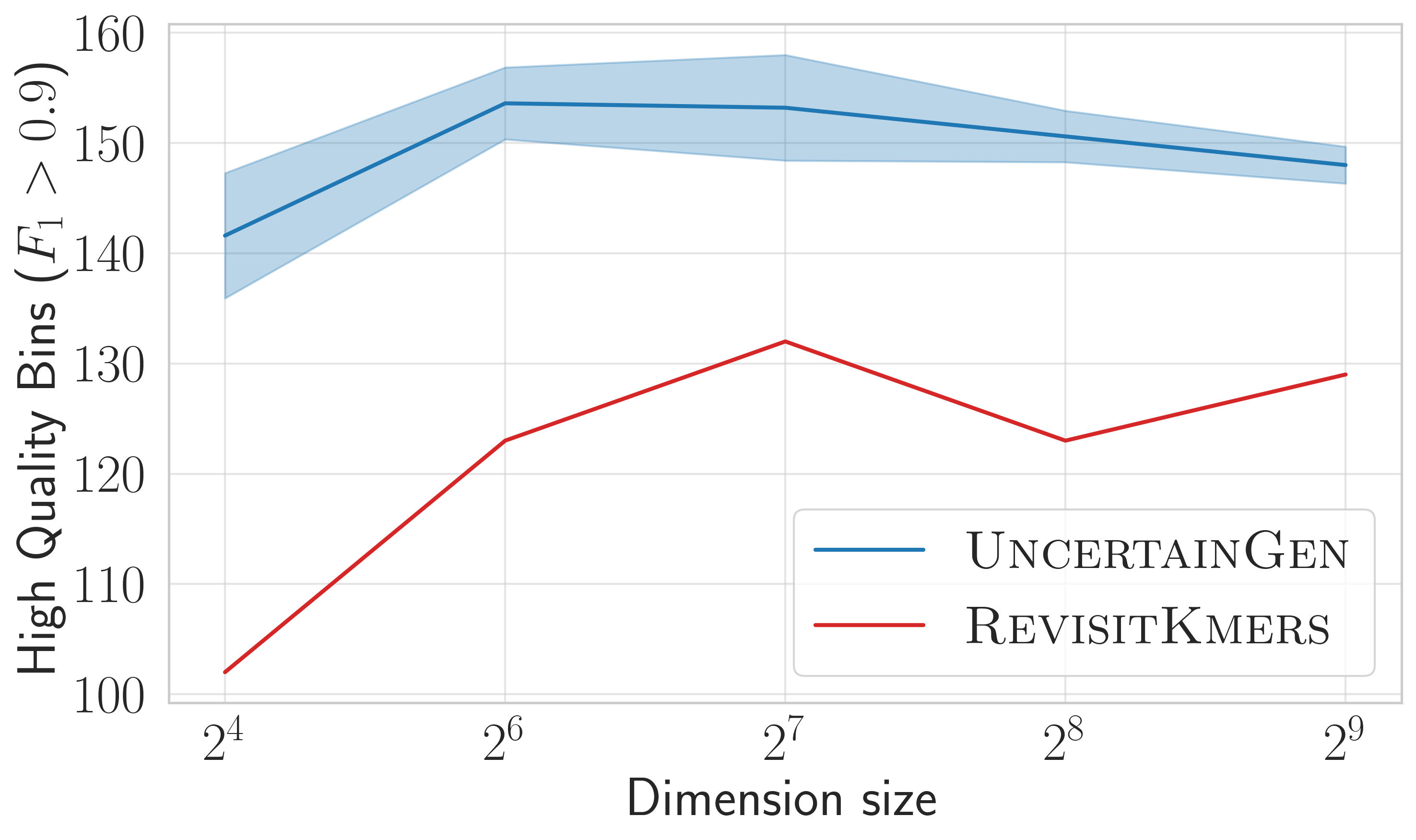}
         \caption{Impact of Dimension Size}
         \label{fig:dimension_size_impact}
     \end{subfigure}
     \caption{Ablation studies examining the behavior of the proposed \textsc{\modelname} model.}
\end{figure*}

\subsection{Ablation Studies}
To better understand the behavior of our model, we perform a series of ablation studies. Due to space constraints, we only discuss the results for the \textsl{Reference} datasets in the main text, but a more comprehensive analysis across all datasets is provided in Appendix \ref{appendix:ablation_studies}.

\textbf{Distribution of variances.} We first analyze the distribution of the predicted variances for sequences in the test data. Specifically, we report $\sum_{d=1}^D\log\left(\phi_\sigma(\mathbf{s})_d + 1\right)$, which aggregates uncertainty across dimensions. As shown in Figure \ref{fig:log_determinant_distribution}, both variants of \textsl{Reference} exhibit multimodal distributions, suggesting that the model captures non-trivial heterogeneity in sequence-level uncertainty. This indicates that the variance estimates are not merely noise but encode meaningful structure about the underlying sequence distributions.

\textbf{Sequence filtering.} We filter out sequences with the largest log-determinant values to examine which sequences the model identifies as uncertain, and we report the number of clusters having a recall score $\geq 0.9$ (Figure \ref{fig:filtering_impact}). For comparison, we also include a random filtering baseline. To ensure fairness, removed sequences are assigned to a "garbage" label so that they do not artificially inflate false negatives. Our results show that filtering by model uncertainty consistently retains a larger number of high-recall clusters compared to random filtering. This shows that the model assigns higher uncertainty to sequences from low-quality bins or to those that would otherwise contribute to false negatives, thereby acting as an effective mechanism for uncertainty-aware sequence selection.

\textbf{Dimension size.} As discussed in the theoretical analysis of Section \ref{sec:model}, incorporating covariance terms provides additional representational capacity to our model, \textsc{\modelname}, compared to its Euclidean distance version, \textsc{RevisitKMers}. Beyond capturing predictive uncertainty, the covariance-dependent Mahalanobis term induces a data-adaptive, non-Euclidean geometry in the embedding space, which allows the model to better separate complex sequence structures. Empirical results in Figure \ref{fig:dimension_size_impact} support this intuition. We observe consistent improvements when covariance terms are included, with the gains being especially pronounced in lower-dimensional settings where representational bottlenecks are most restrictive.

\section{Conclusions, Limitations and Future Work\label{sec:conclusion}}

We introduced \textsc{\modelname}, the first probabilistic embedding framework for metagenomic binning. Unlike deterministic $k$-mer or genome foundation models, our method maps each DNA fragment to a distribution in latent space, explicitly encoding sequence-level uncertainty. Theoretical analysis showed the limitations of fixed-variance embeddings and how adaptive covariance terms enable rigid pairwise distinguishability, while experiments on both synthetic and real metagenomic data demonstrated consistent gains in binning quality over strong deterministic baselines.
 
Our study also reveals some limitations. We relied on a simplified self-supervised pairing strategy and two-layer networks; more expressive architectures or richer positive/negative sampling schemes may further enhance performance. In addition, although variance terms estimate uncertainty, whether these estimates are calibrated or not remains an open question. Moreover, while the toy experiment directly validates this behavior under controlled ambiguity, the benchmark datasets do not annotate conserved, horizontally transferred, or mobile genomic regions, preventing us from identifying the biological source of high learned variance.

Future work will extend \textsc{\modelname} beyond $k$-mer-based representations, explore hierarchical or non-Gaussian distributions for even richer uncertainty modeling, and \textcolor{black}{integrate our embeddings into end-to-end pipelines for metagenome reconstruction}. We hope this framework stimulates broader adoption of uncertainty-aware representations in computational genomics.

\section*{Acknowledgements}
We thank the reviewers for their constructive feedback and insightful comments. We also thank all members of Albertsen Lab at Aalborg University and Katja Hose for the valuable discussions. This work was supported by a research grant (VIL50093) from VILLUM
FONDEN.

\bibliographystyle{plain}
\bibliography{references}

@article{ccelikkanat2024revisiting,
  title={Revisiting k-mer profile for effective and scalable genome representation learning},
  author={{\c{C}}elikkanat, Abdulkadir and Masegosa, Andres and Nielsen, Thomas},
  journal={Advances in Neural Information Processing Systems},
  volume={37},
  pages={118930--118952},
  year={2024}
}

@book{vershynin2018high,
  title={High-dimensional probability: An introduction with applications in data science},
  author={Vershynin, Roman},
  volume={47},
  year={2018},
  publisher={Cambridge university press}
}

@article{nguyen2023hyenadna,
  title={Hyenadna: Long-range genomic sequence modeling at single nucleotide resolution},
  author={Nguyen, Eric and Poli, Michael and Faizi, Marjan and Thomas, Armin and Wornow, Michael and Birch-Sykes, Callum and Massaroli, Stefano and Patel, Aman and Rabideau, Clayton and Bengio, Yoshua and others},
  journal={Advances in neural information processing systems},
  volume={36},
  pages={43177--43201},
  year={2023}
}

@article{ji2021dnabert,
  title={DNABERT: pre-trained Bidirectional Encoder Representations from Transformers model for DNA-language in genome},
  author={Ji, Yanrong and Zhou, Zhihan and Liu, Han and Davuluri, Ramana V},
  journal={Bioinformatics},
  volume={37},
  number={15},
  pages={2112--2120},
  year={2021},
  publisher={Oxford University Press}
}

@article{zhou2023dnabert,
  title={Dnabert-2: Efficient foundation model and benchmark for multi-species genome},
  author={Zhou, Zhihan and Ji, Yanrong and Li, Weijian and Dutta, Pratik and Davuluri, Ramana and Liu, Han},
  journal={arXiv preprint arXiv:2306.15006},
  year={2023}
}

@article{zhou2024dnabert,
    author = {Zhou, Zhihan and Wu, Weimin and Ho, Harrison and Wang, Jiayi and Shi, Lizhen and Davuluri, Ramana V and Wang, Zhong and Liu, Han},
    title = {DNABERT-S: pioneering species differentiation with species-aware DNA embeddings},
    journal = {Bioinformatics},
    volume = {41},
    number = {Supplement\_1},
    pages = {i255-i264},
    year = {2025},
    month = {07},
    issn = {1367-4811}
}

@article{khan2025hyperbolic,
  title={Hyperbolic genome embeddings},
  author={Khan, Raiyan R and Chlenski, Philippe and Pe'er, Itsik},
  journal={arXiv preprint arXiv:2507.21648},
  year={2025}
}

@article{kang2019metabat,
  title={MetaBAT 2: an adaptive binning algorithm for robust and efficient genome reconstruction from metagenome assemblies},
  author={Kang, Dongwan D and Li, Feng and Kirton, Edward and Thomas, Ashleigh and Egan, Rob and An, Hong and Wang, Zhong},
  journal={PeerJ},
  volume={7},
  pages={e7359},
  year={2019},
  publisher={PeerJ Inc.}
}

@article{wu2016maxbin,
  title={MaxBin 2.0: an automated binning algorithm to recover genomes from multiple metagenomic datasets},
  author={Wu, Yu-Wei and Simmons, Blake A and Singer, Steven W},
  journal={Bioinformatics},
  volume={32},
  number={4},
  pages={605--607},
  year={2016},
  publisher={Oxford University Press}
}

@article{wang2024effective,
  title={Effective binning of metagenomic contigs using contrastive multi-view representation learning},
  author={Wang, Ziye and You, Ronghui and Han, Haitao and Liu, Wei and Sun, Fengzhu and Zhu, Shanfeng},
  journal={Nature Communications},
  volume={15},
  number={1},
  pages={585},
  year={2024},
  publisher={Nature Publishing Group UK London}
}

@article{nissen2021improved,
  title={Improved metagenome binning and assembly using deep variational autoencoders},
  author={Nissen, Jakob Nybo and Johansen, Joachim and Alles{\o}e, Rosa Lundbye and S{\o}nderby, Casper Kaae and Armenteros, Jose Juan Almagro and Gr{\o}nbech, Christopher Heje and Jensen, Lars Juhl and Nielsen, Henrik Bj{\o}rn and Petersen, Thomas Nordahl and Winther, Ole and others},
  journal={Nature biotechnology},
  volume={39},
  number={5},
  pages={555--560},
  year={2021},
  publisher={Nature Publishing Group US New York}
}

@article{kutuzova2024binning,
  title={Binning meets taxonomy: TaxVAMB improves metagenome binning using bi-modal variational autoencoder},
  author={Kutuzova, Svetlana and Piera, Pau and Nor Nielsen, Knud and Olsen, Nikoline S and Riber, Leise and Gobbi, Alex and Forero-Junco, Laura Milena and Dougherty, Peter Erdmann and Westergaard, Jesper Cairo and Christensen, Svend and others},
  journal={bioRxiv},
  pages={2024--10},
  year={2024},
  publisher={Cold Spring Harbor Laboratory}
}

@article{pan2023semibin2,
  title={SemiBin2: self-supervised contrastive learning leads to better MAGs for short-and long-read sequencing},
  author={Pan, Shaojun and Zhao, Xing-Ming and Coelho, Luis Pedro},
  journal={Bioinformatics},
  volume={39},
  number={Supplement\_1},
  pages={i21--i29},
  year={2023},
  publisher={Oxford University Press}
}

@article{muzellec2018generalizing,
  title={Generalizing point embeddings using the wasserstein space of elliptical distributions},
  author={Muzellec, Boris and Cuturi, Marco},
  journal={Advances in Neural Information Processing Systems},
  volume={31},
  year={2018}
}

@inproceedings{ohModelingUncertaintyHedged2019,
  title = {Modeling {{Uncertainty}} with {{Hedged Instance Embeddings}}},
  booktitle = {International {{Conference}} on {{Learning Representations}}},
  author = {Oh, Seong Joon and Murphy, Kevin P. and Pan, Jiyan and Roth, Joseph and Schroff, Florian and Gallagher, Andrew C.},
  year = {2019},
  month = sep,
  urldate = {2025-09-19},
  langid = {english}
}

@inproceedings{shiProbabilisticFaceEmbeddings2019a,
  author = {Shi, Yichun and Jain, Anil},
  title = {Probabilistic {{Face Embeddings}}},
  booktitle = {2019 {{IEEE}}/{{CVF International Conference}} on {{Computer Vision}} ({{ICCV}})},
  year = {2019},
  month = oct,
  pages = {6901--6910},
  publisher = {IEEE},
  address = {Seoul, Korea (South)},
  doi = {10.1109/ICCV.2019.00700},
  urldate = {2025-09-19},
  copyright = {https://ieeexplore.ieee.org/Xplorehelp/downloads/license-information/IEEE.html},
  isbn = {978-1-7281-4803-8},
  langid = {english}
}

@article{warburgBayesianMetricLearning2023,
  title = {Bayesian {{Metric Learning}} for {{Uncertainty Quantification}} in {{Image Retrieval}}},
  author = {Warburg, Frederik and Miani, Marco and Brack, Silas and Hauberg, S{\o}ren},
  year = {2023},
  month = dec,
  journal = {Advances in Neural Information Processing Systems},
  volume = {36},
  pages = {69178--69190},
  urldate = {2025-01-24},
  langid = {english}
}

@article{karpukhinProbabilisticEmbeddingsRevisited2024,
  title = {Probabilistic Embeddings Revisited},
  author = {Karpukhin, Ivan and Dereka, Stanislav and Kolesnikov, Sergey},
  year = {2024},
  month = jun,
  journal = {The Visual Computer},
  volume = {40},
  number = {6},
  pages = {4373--4386},
  issn = {1432-2315},
  doi = {10.1007/s00371-023-03087-3},
  urldate = {2025-09-19},
  langid = {english}
}

@inproceedings{jeongProbabilisticVariationalContrastive2025,
  title={Probabilistic Variational Contrastive Learning},
  author={Jeong, Minoh and Kim, Seonho and Hero, Alfred O},
  booktitle={NeurIPS 2025 Workshop on Structured Probabilistic Inference \& Generative Modeling},
year = {2025},
}

@misc{bansalUnderstandingSelfSupervisedLearning2025,
  title = {Understanding {{Self-Supervised Learning}} via {{Gaussian Mixture Models}}},
  author = {Bansal, Parikshit and Kavis, Ali and Sanghavi, Sujay},
  year = {2025},
  month = feb,
  number = {arXiv:2411.03517},
  eprint = {2411.03517},
  primaryclass = {cs},
  publisher = {arXiv},
  doi = {10.48550/arXiv.2411.03517},
  urldate = {2025-07-30},
  archiveprefix = {arXiv}
}

@inproceedings{kirchhofProbabilisticContrastiveLearning2023,
  title = {Probabilistic {{Contrastive Learning Recovers}} the {{Correct Aleatoric Uncertainty}} of {{Ambiguous Inputs}}},
  booktitle = {Proceedings of the 40th {{International Conference}} on {{Machine Learning}}},
  author = {Kirchhof, Michael and Kasneci, Enkelejda and Oh, Seong Joon},
  year = {2023},
  month = jul,
  pages = {17085--17104},
  publisher = {PMLR},
  issn = {2640-3498},
  urldate = {2025-02-11},
  langid = {english}
}

@inproceedings{alemiDeepVariationalInformation2017,
  title = {Deep {{Variational Information Bottleneck}}},
  booktitle = {{{ICLR}}},
  author = {Alemi, Alexander A. and Fischer, Ian and Dillon, Joshua V. and Murphy, Kevin},
  year = {2017},
  urldate = {2021-05-07},
  langid = {english}
}

@inproceedings{vilnisWordRepresentationsGaussian2015,
  title = {Word {{Representations}} via {{Gaussian Embedding}}},
  booktitle = {{{ICLR}}},
  author = {Vilnis, Luke and McCallum, Andrew},
  year = {2015},
  month = may,
  eprint = {1412.6623},
  primaryclass = {cs},
  publisher = {arXiv},
  doi = {10.48550/arXiv.1412.6623},
  urldate = {2025-09-20},
  archiveprefix = {arXiv}
}

@article{falkowskiMicrobialEnginesThat2008,
  title = {The {{Microbial Engines That Drive Earth}}'s {{Biogeochemical Cycles}}},
  author = {Falkowski, Paul G. and Fenchel, Tom and Delong, Edward F.},
  year = {2008},
  month = may,
  journal = {Science},
  volume = {320},
  number = {5879},
  pages = {1034--1039},
  publisher = {American Association for the Advancement of Science},
  doi = {10.1126/science.1153213},
  urldate = {2025-09-24}
}

@article{timmisContributionMicrobialBiotechnology2017,
  title = {The Contribution of Microbial Biotechnology to Sustainable Development Goals},
  author = {Timmis, Kenneth and {de Vos}, Willem M. and Ramos, Juan Luis and Vlaeminck, Siegfried E. and Prieto, Auxiliadora and Danchin, Antoine and Verstraete, Willy and {de Lorenzo}, Victor and Lee, Sang Yup and Br{\"u}ssow, Harald and Timmis, James Kenneth and Singh, Brajesh K.},
  year = {2017},
  month = sep,
  journal = {Microbial Biotechnology},
  volume = {10},
  number = {5},
  pages = {984--987},
  issn = {1751-7915},
  doi = {10.1111/1751-7915.12818},
  langid = {english},
  pmcid = {PMC5609250},
  pmid = {28840974}
}

@article{cavicchioliScientistsWarningHumanity2019,
  title = {Scientists' Warning to Humanity: Microorganisms and Climate Change},
  shorttitle = {Scientists' Warning to Humanity},
  author = {Cavicchioli, Ricardo and others},
  year = {2019},
  month = sep,
  journal = {Nature Reviews Microbiology},
  volume = {17},
  number = {9},
  pages = {569--586},
  publisher = {Nature Publishing Group},
  issn = {1740-1534},
  doi = {10.1038/s41579-019-0222-5},
  urldate = {2025-09-24},
  copyright = {2019 Springer Nature Limited},
  langid = {english}
}

@article{kuninBioinformaticiansGuideMetagenomics2008,
  title = {A {{Bioinformatician}}'s {{Guide}} to {{Metagenomics}}},
  author = {Kunin, Victor and Copeland, Alex and Lapidus, Alla and Mavromatis, Konstantinos and Hugenholtz, Philip},
  year = {2008},
  month = dec,
  journal = {Microbiology and Molecular Biology Reviews},
  volume = {72},
  number = {4},
  pages = {557--578},
  publisher = {American Society for Microbiology},
  doi = {10.1128/mmbr.00009-08},
  urldate = {2025-09-24}
}

@article{tempertonMetagenomicsMicrobialDiversity2012,
  title = {Metagenomics: Microbial Diversity through a Scratched Lens},
  shorttitle = {Metagenomics},
  author = {Temperton, Ben and Giovannoni, Stephen J.},
  year = {2012},
  month = oct,
  journal = {Current Opinion in Microbiology},
  volume = {15},
  number = {5},
  pages = {605--612},
  issn = {1879-0364},
  doi = {10.1016/j.mib.2012.07.001},
  langid = {english},
  pmid = {22831844}
}

@article{meyerCriticalAssessmentMetagenome2022,
  title = {Critical {{Assessment}} of {{Metagenome Interpretation}}: The Second Round of Challenges},
  shorttitle = {Critical {{Assessment}} of {{Metagenome Interpretation}}},
  author={Meyer, Fernando and others},
  year = {2022},
  month = apr,
  journal = {Nature Methods},
  volume = {19},
  number = {4},
  pages = {429--440},
  publisher = {Nature Publishing Group},
  issn = {1548-7105},
  doi = {10.1038/s41592-022-01431-4},
  urldate = {2025-09-24},
  copyright = {2022 The Author(s)},
  langid = {english}
}

@article{teelingApplicationTetranucleotideFrequencies2004,
  title = {Application of Tetranucleotide Frequencies for the Assignment of Genomic Fragments},
  author = {Teeling, Hanno and Meyerdierks, Anke and Bauer, Margarete and Amann, Rudolf and Gl{\"o}ckner, Frank Oliver},
  year = {2004},
  journal = {Environmental Microbiology},
  volume = {6},
  number = {9},
  pages = {938--947},
  issn = {1462-2920},
  doi = {10.1111/j.1462-2920.2004.00624.x},
  urldate = {2025-09-24},
  langid = {english}
}

@article{chanBinningSequencesUsing2008,
  title = {Binning Sequences Using Very Sparse Labels within a Metagenome},
  author = {Chan, Chon-Kit Kenneth and Hsu, Arthur L. and Halgamuge, Saman K. and Tang, Sen-Lin},
  year = {2008},
  month = apr,
  journal = {BMC Bioinformatics},
  volume = {9},
  number = {1},
  pages = {215},
  issn = {1471-2105},
  doi = {10.1186/1471-2105-9-215},
  urldate = {2025-09-24},
  langid = {english}
}

@article{arnoldHorizontalGeneTransfer2022,
  title = {Horizontal Gene Transfer and Adaptive Evolution in Bacteria},
  author = {Arnold, Brian J. and Huang, I.-Ting and Hanage, William P.},
  year = {2022},
  month = apr,
  journal = {Nature Reviews Microbiology},
  volume = {20},
  number = {4},
  pages = {206--218},
  publisher = {Nature Publishing Group},
  issn = {1740-1534},
  doi = {10.1038/s41579-021-00650-4},
  urldate = {2025-09-24},
  copyright = {2021 Springer Nature Limited},
  langid = {english}
}

\clearpage
\appendix
\section{Appendix\label{sec:appendix}}

\subsection{The role of $\alpha$ in $\mathbf{K}_{ij}=\alpha(\mathbf{S}_i + \mathbf{S}_j)$}
\label{sec:alpha}
Consider the closed form expectation of Lemma~\ref{lemma:closed_form_expectation}:
\begin{align}\nonumber
&\mathbb{E}_{\substack{\mathbf{z}_i \sim \mathcal{N}\left(\boldsymbol{\mu}_i, \mathbf{S}_i \right) \\ \mathbf{z}_j \sim \mathcal{N}\left(\boldsymbol{\mu}_j, \mathbf{S}_j \right)}}\!\!\left[ \exp\left(\!\! - \frac{1}{2}(\mathbf{z}_i \!- \mathbf{z}_j)^{\top}\mathbf{K}_{ij}^{-1}(\mathbf{z}_i \!- \mathbf{z}_j) + \beta \right)\! \right]
\\
&\qquad\qquad\qquad\qquad\qquad\qquad= \frac{ 1 }{\sqrt{ |\mathbf{K}_{ij}^{-1}\mathbf{S}_{ij} \!+ \mathbf{I}| }} \exp\!\left( \!-\frac{1}{2}\boldsymbol{\mu}_{ij}^{\top}\left(  \mathbf{S}_{ij} \!+\! \mathbf{K}_{ij}\right)^{-1}\!\!\boldsymbol{\mu}_{ij} + \beta \right).
\end{align}
where $\boldsymbol{\mu}_{ij} := \boldsymbol{\mu}_{i} - \boldsymbol{\mu}_j$ and $\mathbf{S}_{ij} := \mathbf{S}_i + \mathbf{S}_j$. By setting $\mathbf{K}_{ij}= \alpha(\mathbf{S}_i + \mathbf{S}_j) = \alpha\mathbf{S}_{ij}$, then we have
\begin{align}
    \frac{ 1 }{\sqrt{ |\mathbf{K}_{ij}^{-1}\mathbf{S}_{ij} + \mathbf{I}| }} = \frac{ 1 }{\sqrt{ |\alpha^{-1} \mathbf{I} + \mathbf{I}| }} =\frac{ 1 }{\sqrt{ |(1+\alpha^{-1}) \mathbf{I}| }} =(1+\alpha^{-1})^{-D/2} 
\end{align}
and
\begin{align}
 -\frac{1}{2}\boldsymbol{\mu}_{ij}^{\top}\left(  \mathbf{S}_{ij} + \mathbf{K}_{ij}\right)^{-1}\boldsymbol{\mu}_{ij}
 &=  -\frac{1}{2}\boldsymbol{\mu}_{ij}^{\top}\left(  \mathbf{S}_{ij} + \alpha\mathbf{S}_{ij}\right)^{-1}\boldsymbol{\mu}_{ij}\\
 &=  -\frac{1}{2}\boldsymbol{\mu}_{ij}^{\top}\left(  (1+\alpha)\mathbf{S}_{ij}\right)^{-1}\boldsymbol{\mu}_{ij}\\
 &=  -\frac{1}{2(1+\alpha)}\boldsymbol{\mu}_{ij}^{\top}\mathbf{S}_{ij}^{-1}\boldsymbol{\mu}_{ij}.
\end{align}
To isolate the effect of $\alpha$, we set $\beta=0$ in the following analysis. Hence,
\begin{align}
&\mathbb{E}_{\substack{\mathbf{z}_i \sim \mathcal{N}\left(\boldsymbol{\mu}_i, \mathbf{S}_i \right) \\ \mathbf{z}_j \sim \mathcal{N}\left(\boldsymbol{\mu}_j, \mathbf{S}_j \right)}}\!\!\left[ \exp\left(\!\! - \frac{1}{2}(\mathbf{z}_i \!- \mathbf{z}_j)^{\top}\mathbf{K}_{ij}^{-1}(\mathbf{z}_i \!- \mathbf{z}_j) + \beta \right)\! \right]
\\
&\qquad\qquad\qquad\qquad\qquad\qquad= \frac{ 1 }{\sqrt{ |\mathbf{K}_{ij}^{-1}\mathbf{S}_{ij} \!+ \mathbf{I}| }} \exp\!\left( \!-\frac{1}{2}\boldsymbol{\mu}_{ij}^{\top}\left(  \mathbf{S}_{ij} \!+\! \mathbf{K}_{ij}\right)^{-1}\!\!\boldsymbol{\mu}_{ij} \right)
\\
&\qquad\qquad\qquad\qquad\qquad\qquad= (1+\alpha^{-1})^{-D/2} \exp \left (-\frac{1}{2(1+\alpha)}\boldsymbol{\mu}_{ij}^{\top}  \mathbf{S}_{ij}^{-1}\boldsymbol{\mu}_{ij}  \right )
\\
&\qquad\qquad\qquad\qquad\qquad\qquad\rightarrow \begin{cases}
1 & \text{ as } \alpha\rightarrow \infty \\
0 & \text{ as } \alpha \rightarrow 0
\end{cases}.
\end{align}

\subsection{Theoretical Analysis}\label{app:proofs}

\begin{relemma}[\textbf{\ref{lemma:closed_form_expectation}}](Closed-form expectation)
Let $\mathbf{z}_i \sim \mathcal{N}\left( \boldsymbol{\mu}_i, \mathbf{S}_i \right)$ and $\mathbf{z}_j \sim \mathcal{N}\left( \boldsymbol{\mu}_j, \mathbf{S}_j \right)$ be independent random variables. For a given positive definite matrix $\mathbf{K}_{ij} \succ 0$, the expectation is equal to
\begin{align}
\mathbb{E}_{\substack{\mathbf{z}_{i} \sim \mathcal{N}(\boldsymbol{\mu}_{i}, \mathbf{S}_{i}) \\ \mathbf{z}_{j} \sim \mathcal{N}(\boldsymbol{\mu}_{j}, \mathbf{S}_{j})}}\left[ p(Y_{ij}=1 \mid \mathbf{z}_i, \mathbf{z}_j )\right] &= \mathbb{E}_{\substack{\mathbf{z}_i \sim \mathcal{N}\left(\boldsymbol{\mu}_i, \mathbf{S}_i \right) \\ \mathbf{z}_j \sim \mathcal{N}\left(\boldsymbol{\mu}_j, \mathbf{S}_j \right)}}\left[ \exp\left(- \frac{1}{2}(\mathbf{z}_i - \mathbf{z}_j)^{\top}\mathbf{K}_{ij}^{-1}(\mathbf{z}_i - \mathbf{z}_j) + \beta \right) \right]\nonumber
\\
&= \frac{ 1 }{\sqrt{ |\mathbf{K}_{ij}^{-1}\mathbf{S}_{ij} + \mathbf{I}| }}  \exp\left( -\frac{1}{2}\boldsymbol{\mu}_{ij}^{\top}\left(  \mathbf{S}_{ij} \!+\! \mathbf{K}_{ij}\right)^{-1}\!\boldsymbol{\mu}_{ij} + \beta \right)
\end{align}
where $\boldsymbol{\mu}_{ij} := \boldsymbol{\mu}_{i} - \boldsymbol{\mu}_j$ and $\mathbf{S}_{ij} := \mathbf{S}_i + \mathbf{S}_j$.
\end{relemma}

\begin{proof}
Since $\mathbf{z}_i \sim \mathcal{N}\left( \boldsymbol{\mu}_i, \mathbf{S}_i \right)$ and $\mathbf{z}_j \sim \mathcal{N}\left( \boldsymbol{\mu}_j, \mathbf{S}_j \right)$ are independent random variables, the difference, $\mathbf{z}_i- \mathbf{z}_j$ is also normally distributed so we can write $\mathbf{z}_i - \mathbf{z}_j \sim \mathcal{N}\left(\boldsymbol{\mu}_i-\boldsymbol{\mu}_j, \mathbf{S}_i+\mathbf{S}_j \right)$. In other words, $\mathbf{z}_{ij} \sim \mathcal{N}\left(\boldsymbol{\mu}_{ij}, \mathbf{S}_{ij} \right)$ where $\boldsymbol{\mu}_{ij} := \boldsymbol{\mu}_i-\boldsymbol{\mu}_j$, $\mathbf{z}_{ij} := \mathbf{z}_i - \mathbf{z}_j$, and $\mathbf{S}_{ij} := \mathbf{S}_i + \mathbf{S}_j$, then we can write the expected term as follows:
\begin{align}
& \mathbb{E}_{\substack{\mathbf{z}_i \sim \mathcal{N}\left(\boldsymbol{\mu}_i, \mathbf{S}_i \right) \\ \mathbf{z}_j \sim \mathcal{N}\left(\boldsymbol{\mu}_j, \mathbf{S}_j \right)}}\left[ \exp\left(- \frac{1}{2}(\mathbf{z}_i - \mathbf{z}_j)^{\top}\mathbf{K}_{ij}^{-1}(\mathbf{z}_i - \mathbf{z}_j) + \beta \right) \right]\label{eq:zij_expectation}
\\
&\quad = \mathbb{E}_{\mathbf{z}_i-\mathbf{z}_j \sim  \mathcal{N}\left(\boldsymbol{\mu}_{ij}, \mathbf{S}_{ij} \right)}\left[ \exp\left(- \frac{1}{2}(\mathbf{z}_i - \mathbf{z}_j)^{\top}\mathbf{K}_{ij}^{-1}(\mathbf{z}_i - \mathbf{z}_j) + \beta \right) \right]
\\
&\quad = \mathbb{E}_{\mathbf{z}_{ij} \sim  \mathcal{N}\left(\boldsymbol{\mu}_{ij}, \mathbf{S}{ij} \right)}\left[ \exp\left(- \frac{1}{2}\mathbf{z}_{ij}^{\top}\mathbf{K}_{ij}^{-1}\mathbf{z}_{ij} + \beta \right) \right]
\\
&\quad = \int \frac{1}{\sqrt{(2\pi)^D|\mathbf{S}_{ij}|}}\exp\left(\! -\frac{1}{2}(\mathbf{z}_{ij} \!-\! \boldsymbol{\mu}_{ij})^{\top}\mathbf{S}_{ij}^{-1}(\mathbf{z}_{ij} \!-\! \boldsymbol{\mu}_{ij})\! \right)\exp\left(\!-\frac{1}{2} \mathbf{z}_{ij}^{\top}\mathbf{K}_{ij}^{-1}\mathbf{z}_{ij} + \beta \right) \mathrm{d}{\mathbf{z}_{ij}}
\\
&\quad = \frac{\exp(\beta)}{\sqrt{(2\pi)^D|\mathbf{S}_{ij}|}}\int \exp\left( -\frac{1}{2}\left( (\mathbf{z}_{ij} - \boldsymbol{\mu}_{ij})^{\top}\mathbf{S}_{ij}^{-1}(\mathbf{z}_{ij} - \boldsymbol{\mu}_{ij}) \right) + \mathbf{z}_{ij}^{\top}\mathbf{K}_{ij}^{-1}\mathbf{z}_{ij} \right) \mathrm{d}{\mathbf{z}_{ij}}
\end{align}

By expanding and regrouping the terms in the integral, we can write that:
\begin{align}
 &\int \exp\left( -\frac{1}{2}\left( (\mathbf{z}_{ij} - \boldsymbol{\mu}_{ij})^{\top}\mathbf{S}_{ij}^{-1}(\mathbf{z}_{ij} - \boldsymbol{\mu}_{ij}) + \mathbf{z}_{ij}^{\top}\mathbf{K}_{ij}^{-1}\mathbf{z}_{ij}\right) \right) \mathrm{d}{\mathbf{z}_{ij}}\label{eq:integral_deriv_step}
\\
&\quad= \int \exp\left( -\frac{1}{2}\left( \mathbf{z}_{ij}^{\top}\mathbf{K}_{ij}^{-1}\mathbf{z}_{ij} + \mathbf{z}_{ij}^{\top}\mathbf{S}_{ij}^{-1}\mathbf{z}_{ij} -2 \mathbf{z}_{ij}^{\top}\mathbf{S}_{ij}^{-1}\mu_{ij} +\mu_{ij}^{\top}\mathbf{S}_{ij}^{-1}\mu_{ij} \right) \right) \mathrm{d}{\mathbf{z}_{ij}}
\\
&\quad= \int \exp\left( -\frac{1}{2}\left( \mathbf{z}_{ij}^{\top}(\mathbf{K}_{ij}^{-1} + \mathbf{S}_{ij}^{-1})\mathbf{z}_{ij} -2 \mathbf{z}_{ij}^{\top}\mathbf{S}_{ij}^{-1}\boldsymbol{\mu}_{ij} + \boldsymbol{\mu}_{ij}^{\top}\mathbf{S}_{ij}^{-1}\mu_{ij} \right) \right) \mathrm{d}{\mathbf{z}_{ij}}
\\
&\quad= \int \exp\left( -\frac{1}{2}\left( \mathbf{z}_{ij}^{\top}\mathbf{A}_{ij}\mathbf{z}_{ij} -2 \mathbf{z}_{ij}^{\top}\mathbf{S}_{ij}^{-1}\mu_{ij} + \mu_{ij}^{\top}\mathbf{S}_{ij}^{-1}\mu_{ij} \right) \right) \mathrm{d}{\mathbf{z}_{ij}}
\\
&\quad= \int \exp\bigg( -\frac{1}{2}\Big( ( \mathbf{z}_{ij} -  \mathbf{A}_{ij}^{-1}\mathbf{S}_{ij}^{-1}\boldsymbol{\mu}_{ij} )^{\top}\mathbf{A}_{ij} ( \mathbf{z}_{ij} - \mathbf{A}_{ij}^{-1}\mathbf{S}_{ij}^{-1}\boldsymbol{\mu}_{ij} ) -\nonumber
\\
&\qquad\qquad\qquad\qquad\qquad\qquad\qquad\qquad\qquad\quad \boldsymbol{\mu}_{ij}^{\top}\mathbf{S}_{ij}^{-1}\mathbf{A}_{ij}^{-1}\mathbf{S}_{ij}^{-1}\boldsymbol{\mu}_{ij}  +  \boldsymbol{\mu}_{ij}^{\top}\mathbf{S}_{ij}^{-1}\boldsymbol{\mu}_{ij} \Big) \bigg) \mathrm{d}{\mathbf{z}_{ij}}\label{eq:integral_split}
\\
&\quad= \sqrt{(2\pi)^D |\mathbf{A}_{ij}^{-1}| } \exp\left( -\frac{1}{2}\left(  - \mu_{ij}^{\top}\mathbf{S}_{ij}^{-1}\mathbf{A}_{ij}^{-1}\mathbf{S}_{ij}^{-1}\mu_{ij}  +  \boldsymbol{\mu}_{ij}^{\top}\mathbf{S}_{ij}^{-1}\boldsymbol{\mu}_{ij} \right) \right)
\end{align}
where $\mathbf{A}_{ij} := \mathbf{K}_{ij}^{-1} + \mathbf{S}_{ij}^{-1}$. In Eq. \ref{eq:integral_split}, we add and substract the term $\boldsymbol{\mu}_{ij}^{\top}\mathbf{S}_{ij}^{-1}\mathbf{A}_{ij}^{-1}\mathbf{S}_{ij}^{-1}\boldsymbol{\mu}_{ij}$ so that the first component depends only on $\mathbf{z}_{ij}$. It also corresponds to the numerator of a normal distribution with mean $\mathbf{A}_{ij}^{-1}\mathbf{S}_{ij}^{-1}\boldsymbol{\mu}_{ij}$ and covariance $\mathbf{A}_{ij}^{-1}$. Hence, the last equality follows from the standard Gaussian integral:
\begin{align*}
& \int \exp\left( -\frac{1}{2} ( \mathbf{z}_{ij} -  \mathbf{A}_{ij}^{-1}\mathbf{S}_{ij}^{-1}\boldsymbol{\mu}_{ij} )^{\top}\mathbf{A}_{ij} ( \mathbf{z}_{ij} - \mathbf{A}_{ij}^{-1}\mathbf{S}_{ij}^{-1}\boldsymbol{\mu}_{ij} )  \right)\mathrm{d}\mathbf{z}_{ij} = \sqrt{(2\pi)^D \left|\mathbf{A}_{ij}^{-1}\right| }
\end{align*}

Therefore, the expectation term in Eq. \ref{eq:zij_expectation} can then be rewritten as follows:
\begin{align}
&\mathbb{E}_{\substack{\mathbf{z}_i \sim \mathcal{N}\left(\boldsymbol{\mu}_i, \mathbf{S}_i \right) \\ \mathbf{z}_j \sim \mathcal{N}\left(\boldsymbol{\mu}_j, \mathbf{S}_j \right)}}\left[ \exp\left(- \frac{1}{2}(\mathbf{z}_i - \mathbf{z}_j)^{\top}\mathbf{K}_{ij}^{-1}(\mathbf{z}_i - \mathbf{z}_j) + \beta \right) \right]
\\
&\quad = \frac{  \exp(\beta) }{\sqrt{(2\pi)^D|\mathbf{S}_{ij}|}}\int \exp\left( -\frac{1}{2}\left( (\mathbf{z}_{ij} - \boldsymbol{\mu}_{ij})^{\top}\mathbf{S}_{ij}^{-1}(\mathbf{z}_{ij} - \boldsymbol{\mu}_{ij}) + \mathbf{z}_{ij}^{\top}\mathbf{K}_{ij}^{-1}\mathbf{z}_{ij}\right) \right) \mathrm{d}{\mathbf{z}_{ij}}
\\
&\qquad = \frac{  \exp(\beta) }{\sqrt{(2\pi)^D|\mathbf{S}_{ij}|}} \sqrt{(2\pi)^D |\mathbf{A}_{ij}^{-1}| } \exp\left( -\frac{1}{2}\left(  - \boldsymbol{\mu}_{ij}^{\top}\mathbf{S}_{ij}^{-1}\mathbf{A}_{ij}^{-1}\mathbf{S}_{ij}^{-1}\mu_{ij}  +  \mu_{ij}^{\top}\mathbf{S}_{ij}^{-1}\mu_{ij} \right) \right)
\\
&\qquad = \frac{ \exp(\beta) }{\sqrt{ |\mathbf{A}_{ij}| |\mathbf{S}_{ij}|}}  \exp\left( -\frac{1}{2}\left(  - \mu_{ij}^{\top}\mathbf{S}_{ij}^{-1}\mathbf{A}_{ij}^{-1}\mathbf{S}_{ij}^{-1}\mu_{ij}  +  \mu_{ij}^{\top}\mathbf{S}_{ij}^{-1}\mu_{ij} \right) \right)
\end{align}

where $|\mathbf{A}_{ij}^{-1}| = |\mathbf{A}_{ij}|^{-1}$. By substituting $\mathbf{A}_{ij}$ with $(\mathbf{K}_{ij}^{-1} + \mathbf{S}_{ij}^{-1})$ and applying the Woodbury Matrix Identity, we can obtain
\begin{align}
&\mathbb{E}_{\substack{\mathbf{z}_i \sim \mathcal{N}\left(\boldsymbol{\mu}_i, \mathbf{S}_i \right) \\ \mathbf{z}_j \sim \mathcal{N}\left(\boldsymbol{\mu}_j, \mathbf{S}_j \right)}}\left[ \exp\left(- \frac{1}{2}(\mathbf{z}_i - \mathbf{z}_j)^{\top}\mathbf{K}_{ij}^{-1}(\mathbf{z}_i - \mathbf{z}_j)  + \beta \right) \right]
\\
&\quad = \frac{ \exp(\beta) }{\sqrt{ |\mathbf{A}_{ij}| |\mathbf{S}_{ij}|}}  \exp\left( -\frac{1}{2}\left(  - \mu_{ij}^{\top}\mathbf{S}_{ij}^{-1}\mathbf{A}_{ij}^{-1}\mathbf{S}_{ij}^{-1}\mu_{ij}  +  \mu_{ij}^{\top}\mathbf{S}_{ij}^{-1}\mu_{ij} \right) \right)
\\
&\quad = \frac{ \exp(\beta) }{\sqrt{ |\mathbf{K}_{ij}^{-1} + \mathbf{S}_{ij}^{-1}| |\mathbf{S}_{ij}|}}  \exp\left( -\frac{1}{2}\left(  - \mu_{ij}^{\top}\mathbf{S}_{ij}^{-1}(\mathbf{K}_{ij}^{-1} + \mathbf{S}_{ij}^{-1})^{-1}\mathbf{S}_{ij}^{-1}\mu_{ij}  +  \mu_{ij}^{\top}\mathbf{S}_{ij}^{-1}\mu_{ij} \right) \right)
\\
&\quad = \frac{ \exp(\beta) }{\sqrt{ |\mathbf{K}_{ij}^{-1}\mathbf{S}_{ij} + \mathbf{I}| }}  \exp\left( -\frac{1}{2}\mu_{ij}^{\top}\left(  \mathbf{S}_{ij}^{-1} - \mathbf{S}_{ij}^{-1}(\mathbf{K}_{ij}^{-1} + \mathbf{S}_{ij}^{-1})^{-1}\mathbf{S}_{ij}^{-1} \right)\mu_{ij} \right)
\\
&\quad = \frac{ 1 }{\sqrt{ |\mathbf{K}_{ij}^{-1}\mathbf{S}_{ij} + \mathbf{I}| }}  \exp\left( -\frac{1}{2}\mu_{ij}^{\top}\left(  \mathbf{S}_{ij} + \mathbf{K}_{ij}\right)^{-1}\mu_{ij} + \beta \right)
\end{align}
Finally, by substituting the terms, $\boldsymbol{\mu}_{ij}$ and $\mathbf{S}_{ij}$, we can conclude that
\begin{align}
&\mathbb{E}_{\substack{\mathbf{z}_i \sim \mathcal{N}\left(\boldsymbol{\mu}_i, \mathbf{S}_i \right) \\ \mathbf{z}_j \sim \mathcal{N}\left(\boldsymbol{\mu}_j, \mathbf{S}_j \right)}}\left[ \exp\left(- \frac{1}{2}(\mathbf{z}_i - \mathbf{z}_j)^{\top}\mathbf{K}_{ij}^{-1}(\mathbf{z}_i - \mathbf{z}_j) + \beta \right) \right]
\\
&\quad = \frac{ 1 }{\sqrt{ |\mathbf{K}_{ij}^{-1}(\mathbf{S}_{i}+\mathbf{S}_{j}) + \mathbf{I}| }}  \exp\left( -\frac{1}{2}(\boldsymbol{\mu}_{i} - \boldsymbol{\mu}_j)^{\top}\left(  \mathbf{S}_{i}+\mathbf{S}_{j} + \mathbf{K}_{ij}\right)^{-1}(\boldsymbol{\mu}_{i} - \boldsymbol{\mu}_j) + \beta \right)
\end{align}
\end{proof}

\begin{relemma}[\textbf{\ref{lemma:bounds}}]
Let $\epsilon \in (0, 1/2 )$, and let $\phi = (\phi_{\mu}, \phi_{\sigma}): \mathcal{S} \to  \mathbb{R}^D \times \mathbb{R}^D_{+}$ be an $\epsilon$-distinguishable embedding function for a triplet $(\mathbf{s}_i,\mathbf{s}_j,y_{ij})\in \mathcal{S}^2 \times \{0,1\}$. Then the following bounds hold:
\begin{align}
&\min_{d}\left\{ (\phi_\sigma(\mathbf{s}_i) + \phi_\sigma(\mathbf{s}_j))_d \right\}
 \log\left( \frac{1}{ \epsilon^4 }  \right) \leq \|\phi_\mu(\mathbf{s}_i) - \phi_\mu(\mathbf{s}_j)\|^2_2 &&\!\!\!\!\!\!\!\!\!\! \text{if $y_{ij}=0$,}
\\
&\max_{d}\left\{ (\phi_\sigma(\mathbf{s}_i) + \phi_\sigma(\mathbf{s}_j))_d \right\}
 \log\left( \frac{1}{(1-\epsilon)^{4}} \right) \geq \|\phi_\mu(\mathbf{s}_i) - \phi_\mu(\mathbf{s}_j)\|^2_2 &&\!\!\!\!\!\!\!\!\!\! \text{if $y_{ij}=1$.}
\end{align}
\end{relemma}
\begin{proof}
Let us define $\mathbf{S}_{ij} := \text{diag}(\phi_\sigma(\mathbf{s}_i) + \phi_\sigma(\mathbf{s}_j))$ and $\boldsymbol{\mu}_{ij} = \phi_\mu(\mathbf{s}_i) - \phi_\mu(\mathbf{s}_j) $. We will first establish the lower bound for a negative pair (i.e., $y_{ij}=0$). 
By $\epsilon$-distinguishability, we assume that
$p(Y_{ij}=0\mid \mathbf{s}_i, \mathbf{s}_j)\geq 1-\epsilon$ and it implies that $\exp\left( -\tfrac{1}{4}\,\boldsymbol{\mu}_{ij}^\top \mathbf{S}_{ij}^{-1} \boldsymbol{\mu}_{ij} \right) \leq \epsilon$.  Hence, $\boldsymbol{\mu}_{ij}^\top \mathbf{S}_{ij}^{-1}\boldsymbol{\mu}_{ij} \geq \log\left(\epsilon^{-4}\right)$, and we can obtain
\begin{align}
\min_{d} \{ (\mathbf{S}_{ij})_d \} \log\left( \epsilon^{-4} \right)
&\leq \min_{d}\{ (\mathbf{S}_{ij})_d \}\,\boldsymbol{\mu}_{ij}^\top \mathbf{S}_{ij}^{-1}\boldsymbol{\mu}_{ij}
\\
&= \min_{d}\{ (\mathbf{S}_{ij})_d \}\,\sum_{d=1}^D \frac{(\boldsymbol{\mu}_{ij})_d^2}{ (\mathbf{S}_{ij})_d }
\\
&\leq \min_{d}\{ (\mathbf{S}_{ij})_d \}
\left( \frac{1}{\min_{d}\{ (\mathbf{S}_{ij})_d \}}\sum_{d=1}^D (\boldsymbol{\mu}_{ij})_d^2 \right)
\\
&= \sum_{d=1}^D (\boldsymbol{\mu}_{i} - \boldsymbol{\mu}_{j})_d^2
\\
&= \|\phi_\mu(\mathbf{s}_i)-\phi_\mu(\mathbf{s}_j)\|^2_2.
\end{align}
Since $\epsilon \in (0,1/2)$, the left-hand side is strictly positive, giving a nontrivial lower bound.

For a positive pair (i.e., $y_{ij}=1$), similarly, we have
$\exp\left( -\frac{1}{4}\,\boldsymbol{\mu}_{ij}^\top \mathbf{S}_{ij}^{-1} \boldsymbol{\mu}_{ij} \right) \geq 1-\epsilon$, which implies $\boldsymbol{\mu}_{ij}^\top \mathbf{S}_{ij}^{-1}\boldsymbol{\mu}_{ij} \leq 4\log\left( \frac{1}{1-\epsilon} \right)$. Then, we can write that
\begin{align}
\max_{d}\{(\mathbf{S}_{ij})_d\}\,\log\left( \frac{1}{(1-\epsilon)^4} \right)
&\geq \max_{d}\{(\mathbf{S}_{ij})_d\}\boldsymbol{\mu}_{ij}^\top \mathbf{S}_{ij}^{-1}\boldsymbol{\mu}_{ij}
\\
&= \max_{d}\{(\mathbf{S}_{ij})_d\}\,\sum_{d=1}^D \frac{(\boldsymbol{\mu}_{ij})_d^2}{(\mathbf{S}_{ij})_d}
\\
&\geq \max_{d}\{(\mathbf{S}_{ij})_d\}\,\left(\frac{1}{\max_{d}\{(\mathbf{S}_{ij})_d\}}\sum_{d=1}^D (\boldsymbol{\mu}_{ij})_d^2 \right)
\\
&= \sum_{d=1}^D (\boldsymbol{\mu}_{i} - \boldsymbol{\mu}_{j})_d^2
\\
&= \|\phi_\mu(\mathbf{s}_i)-\phi_\mu(\mathbf{s}_j)\|^2_2.
\end{align}
This establishes the upper bound for the embedding distance when $y_{ij}=1$.
\end{proof}

\begin{lemma}\label{lemma:packing_number}
Let $\tau > 0$ and $\mathbf{z}_0 \in \mathbb{R}^D$. Then, there exist at most $\mathcal{P}^{\mathcal{D}}_\tau$ distinct points $\{\mathbf{z}_i\}_{i\geq 1} \subset \mathbb{R}^D$ such that
\begin{align}
\| \mathbf{z}_i - \mathbf{z}_0 \| \leq \tau \qquad \text{and} \qquad \| \mathbf{z}_i - \mathbf{z}_j \| \geq \tau \quad \text{for all } 1 \le i < j \le \mathcal{P}^{\mathcal{D}}_\tau,
\end{align}
where $\mathcal{P}^{\mathcal{D}}_\tau$ is the \emph{packing number} of a unit Euclidean ball in $\mathbb{R}^D$.
\end{lemma}
\begin{proof}
By definition, the packing number, $\mathcal{P}^D_\tau$, is the maximal number of points that can fit in $B(\mathbf{z}_0, \tau)$ such that any two points are at least $\tau$ apart. Therefore, any set of points in $B(\mathbf{z}_0, \tau)$ satisfying 
$\| \mathbf{z}_i - \mathbf{z}_j \| \ge \tau$ can contain at most $\mathcal{P}^D_\tau$ points.
\end{proof}

\begin{recorollary}[\textbf{\ref{corollary:limitations}}]
Let $\phi = (\phi_{\mu}, \phi_{\sigma}): \mathcal{S} \to  \mathbb{R}^D \times \mathbb{R}^D_{+}$ be an embedding function with constant variance, such that $\phi_{\sigma}(\mathbf{s})_d = \sigma^2$ is invariant for all $\mathbf{s} \in \mathcal{S}$ and $d \in [D]$. If there exists a collection of $N \geq \mathcal{P}^D_\tau +1$ sequences $\{\mathbf{s}_1, \dots, \mathbf{s}_{N}\} \subseteq \mathcal{S}$ that are each positive pairs with a common sequence $\mathbf{s}_0 \in \mathcal{S}$, but are pairwise negative with respect to each other, then $\phi$ cannot be an $\epsilon$-distinguishable embedding function for any $\epsilon \in (0, 1/2)$.
\end{recorollary}
\begin{proof}
Assume for contradiction that such an $\epsilon$-distinguishable function $\phi: \mathcal{S} \to \mathbb{R}^D \times \mathbb{R}^D_{+}$ exists for some $\epsilon \in (0, 1/2)$ with constant variance, i.e., $\sigma^2 := \phi_{\sigma}(\mathbf{s}_i)_d=\phi_{\sigma}(\mathbf{s}_j)_d$ for all $1 \leq i < j \leq N$ pairs and all $d\in[D]$ so Lemma \ref{lemma:bounds} implies that
\begin{align}
2\sigma^2\log\left( \frac{1}{\epsilon^4} \right)
&\leq \|\phi_\mu(\mathbf{s}_i) - \phi_\mu(\mathbf{s}_j)\|^2_2 
&&\text{if $y_{ij}=0$}
\\
2\sigma^2\log\left( \frac{1}{(1-\epsilon)^4} \right)
&\geq \|\phi_\mu(\mathbf{s}_i) - \phi_\mu(\mathbf{s}_j)\|^2_2 
&&\text{if $y_{ij}=1$.}
\end{align}

for all positive and negative $0 \leq i < j \leq N$ pairs. Note that we have
\begin{align}
\log\left(\frac{1}{\epsilon^4}\right) - \log\left( \frac{1}{(1-\epsilon)^4} \right) = \log\left(\frac{(1-\epsilon)^4}{\epsilon^4}\right) = 4\log\left(  \epsilon^{-1} - 1 \right) > 0   
\end{align}

so let's define $\tau_{\epsilon} := \left( \log(\epsilon^{-4})  + \log((1-\epsilon)^{-4}) \right ) / 2$, then we can write
\begin{align}
\| \phi_{\mu}(\mathbf{s}_{i}) - \phi_{\mu}(\mathbf{s}_{j}) \|_2 > \sigma\sqrt{2\tau_\epsilon},\label{eq:lower_bound_condition}
\end{align}

for all negative $1 \leq i < j \leq N$ pairs and we have
\begin{align}
	\| \phi_{\mu}(\mathbf{s}_{0}) - \phi_{\mu}(\mathbf{s}_{i}) \|_2 < \sigma\sqrt{2\tau_\epsilon}
\end{align}

for all positive pairs $(\mathbf{s}_{0}, \mathbf{s}_{i})$ where $1 \leq i\leq N$. In other words, each $\phi_{\mu}(\mathbf{s}_{i})$ $(i\geq 1)$ lies within a ball of radius $\sigma\sqrt{2\tau_\epsilon}$ centered at $\phi_{\mu}(\mathbf{s}_0)$. However, the condition in Eq. \ref{eq:lower_bound_condition} requires that all negative pairs have to be at least $\sigma\sqrt{2\tau_\epsilon}$ apart from each other at the same time but the maximum number of points that are $\sigma\sqrt{2\tau_\epsilon}$ apart in $B(\phi_{\mu}(\mathbf{s}_{0}), \sigma\sqrt{2\tau_\epsilon})$ is at most $\mathcal{P}^D_\tau$ by Lemma \ref{lemma:packing_number}. Therefore, we obtain a contradiction.
\end{proof}

\begin{retheorem}[\textbf{\ref{thm:theorem}}]
Let $\{\phi^t=(\phi_\mu,\phi_\sigma^t)\}_{t\geq1}$ have fixed bounded means and $\max_d\phi_\sigma^t(\mathbf{s}_i)_d \leq \kappa\min_d\phi_\sigma^t(\mathbf{s}_i)_d$ $\forall\, i\in[N],\,t\ge1$ for some $\kappa < \infty$. Consider $\{\mathbf{s}_i\}_{i=0}^N$, $N\geq2$, such that $(\mathbf{s}_0,\mathbf{s}_i)$ are positive and $(\mathbf{s}_i,\mathbf{s}_j)$ are negative for all $i\neq j\in[N]$, with pairwise distinct mean embeddings. Assume that, for every $d\in[D]$, there exists $i\in[N]$ such that $\phi_\mu(\mathbf{s}_0)_d\neq\phi_\mu(\mathbf{s}_i)_d$. Then $\phi^t$ is eventually $\epsilon$-distinguishable for every $\epsilon\in(0,1/2)$ if and only if
$\max_{i,d}\phi_\sigma^t(\mathbf{s}_i)_d\to0$ and $\min_d\phi_\sigma^t(\mathbf{s}_0)_d\to\infty$.
\end{retheorem}
\begin{proof}
Let $\mathbf{S}_i^t:=\operatorname{diag}(\phi_\sigma^t(\mathbf{s}_i))$ for all $0\leq i\leq N$. Suppose first that $\phi^t$ is eventually $\epsilon$-distinguishable for every $\epsilon\in(0,1/2)$. Then, for every positive pair $(\mathbf{s}_0,\mathbf{s}_i)$ and all sufficiently large $t$,
\begin{align}
Q_{0i}^t := \bigl(\phi_\mu(\mathbf{s}_0)-\phi_\mu(\mathbf{s}_i)\bigr)^\top
(\mathbf{S}_0^t+\mathbf{S}_i^t)^{-1}
\bigl(\phi_\mu(\mathbf{s}_0)-\phi_\mu(\mathbf{s}_i)\bigr)
\leq
\log\!\left((1-\epsilon)^{-4}\right),
\label{eq:positive_limit}
\end{align}
while for every negative pair $(\mathbf{s}_i,\mathbf{s}_j)$,
\begin{align}
Q_{ij}^t := \bigl(\phi_\mu(\mathbf{s}_i)-\phi_\mu(\mathbf{s}_j)\bigr)^\top
(\mathbf{S}_i^t+\mathbf{S}_j^t)^{-1}
\bigl(\phi_\mu(\mathbf{s}_i)-\phi_\mu(\mathbf{s}_j)\bigr)
\geq
\log\!\left(\epsilon^{-4}\right).
\label{eq:negative_limit}
\end{align}
Since these inequalities hold eventually for every $\epsilon\in(0,1/2)$, and
$\log((1-\epsilon)^{-4})\to0$ while $\log(\epsilon^{-4})\to\infty$ as $\epsilon\to0$, it follows that $Q_{0i}^t\to0$ and $Q_{ij}^t\to\infty$. For any $i\in[N]$, choose $j\neq i$. Then
\begin{align}
\frac{\|\phi_\mu(\mathbf{s}_i)-\phi_\mu(\mathbf{s}_j)\|_2^2}
{\min_d(\mathbf{S}_i^t)_d}
\geq
\sum_{d=1}^D
\frac{\bigl(\phi_\mu(\mathbf{s}_i)_d-\phi_\mu(\mathbf{s}_j)_d\bigr)^2}
{(\mathbf{S}_i^t+\mathbf{S}_j^t)_d}
=
Q_{ij}^t
\to\infty.
\end{align}
Since the mean embeddings, $\{\phi_\mu(\mathbf{s}_i)\}_{i\geq1}$, are pairwise distinct, the numerator is fixed, positive, and finite, and therefore $\min_d(\mathbf{S}_i^t)_d\to0$ for every $i\in[N]$.
By uniformly bounded covariance anisotropy, there exists $\kappa<\infty$ such that $\max_d(\mathbf{S}_i^t)_d \leq \kappa\min_d(\mathbf{S}_i^t)_d$,
and hence $\max_{i,d}\phi_\sigma^t(\mathbf{s}_i)_d\to0$.

Now fix $d\in[D]$. By assumption, there exists $i\in[N]$ such that
$\phi_\mu(\mathbf{s}_0)_d\neq\phi_\mu(\mathbf{s}_i)_d$. Since all terms in $Q_{0i}^t$ are nonnegative and $Q_{0i}^t\to0$,
\begin{align}
\frac{\bigl(\phi_\mu(\mathbf{s}_0)_d-\phi_\mu(\mathbf{s}_i)_d\bigr)^2}
{(\mathbf{S}_0^t+\mathbf{S}_i^t)_d}
\to0.
\end{align}
The numerator is fixed and strictly positive, so $(\mathbf{S}_0^t+\mathbf{S}_i^t)_d\to\infty$. Since $(\mathbf{S}_i^t)_d\to0$, it follows that $(\mathbf{S}_0^t)_d\to\infty$. As this holds for every $d\in[D]$, $\min_d\phi_\sigma^t(\mathbf{s}_0)_d\to\infty$.

For the other direction, assume that $\max_{i,d}\phi_\sigma^t(\mathbf{s}_i)_d\to0$ and $\min_d\phi_\sigma^t(\mathbf{s}_0)_d\to\infty$. For every positive pair $(\mathbf{s}_0,\mathbf{s}_i)$,
\begin{align}
Q_{0i}^t
&=
\bigl(\phi_\mu(\mathbf{s}_0)-\phi_\mu(\mathbf{s}_i)\bigr)^\top
(\mathbf{S}_0^t+\mathbf{S}_i^t)^{-1}
\bigl(\phi_\mu(\mathbf{s}_0)-\phi_\mu(\mathbf{s}_i)\bigr)
\\
&\leq
\frac{\|\phi_\mu(\mathbf{s}_0)-\phi_\mu(\mathbf{s}_i)\|_2^2}
{\min_d(\mathbf{S}_0^t)_d}
\to0.
\end{align}
Therefore, for any fixed $\epsilon\in(0,1/2)$ and all sufficiently large $t$, $Q_{0i}^t \leq \log\!\left((1-\epsilon)^{-4}\right)$, which implies
\begin{align}
P_t(Y_{0i}=1\mid\mathbf{s}_0,\mathbf{s}_i)\geq1-\epsilon.
\end{align}
Similarly, for every negative pair $(\mathbf{s}_i,\mathbf{s}_j)$,
\begin{align}
Q_{ij}^t
&=
\bigl(\phi_\mu(\mathbf{s}_i)-\phi_\mu(\mathbf{s}_j)\bigr)^\top
(\mathbf{S}_i^t+\mathbf{S}_j^t)^{-1}
\bigl(\phi_\mu(\mathbf{s}_i)-\phi_\mu(\mathbf{s}_j)\bigr)
\\
&\geq
\frac{\|\phi_\mu(\mathbf{s}_i)-\phi_\mu(\mathbf{s}_j)\|_2^2}
{\max_d(\mathbf{S}_i^t+\mathbf{S}_j^t)_d}
\to\infty,
\end{align}
where the numerator is strictly positive because the mean embeddings, $\{\phi_\mu(\mathbf{s}_i)\}_{i\geq1}$, are pairwise distinct. Hence, for all sufficiently large $t$, $Q_{ij}^t \geq \log\!\left(\epsilon^{-4}\right)$, which implies
\begin{align}
P_t(Y_{ij}=0\mid\mathbf{s}_i,\mathbf{s}_j)\geq1-\epsilon.
\end{align}
Since the collection contains finitely many positive and negative pairs, for every fixed $\epsilon\in(0,1/2)$ there exists a common sufficiently large $t$ for which all pairwise distinguishability conditions hold simultaneously. Therefore, $\phi^t$ is eventually $\epsilon$-distinguishable for every $\epsilon\in(0,1/2)$.
\end{proof}

\section{Experiments}\label{appendix:experiments}

\textbf{Datasets.} For our experiments, we adopt the benchmark datasets introduced in prior work on the metagenomic binning task \citep{zhou2024dnabert}. The datasets are constructed from reference genomes in GenBank and consist of viral, fungal, and bacterial sequences. The training data contains more than $2$ million sequence pairs of length $1000$bp. For testing, we have six datasets (\textsl{Reference 5/6}, \textsl{Plant 5/6}, and \textsl{Marine 5/6}) with species represented by highly variable numbers of sequences ($10$–$4,599$), ranging from $2$-$20$ kbp in length. While \textsl{Reference} datasets consist of DNA fragments from $250$-$330$ fungal and viral genomes, \textsl{Marine} and \textsl{Plant}-associated environments contain $70$k-$125$k sequences from roughly $180$-$520$ species.

\textbf{Baselines.} \textsc{Kmer(cosine)} and \textsc{Kmer($\ell_1$)} are the classical representations obtained by considering the $4$-mer frequency vectors, where sequence similarity is computed using either cosine similarity or an exponential kernel over the $\ell_1$ distance. \textsc{HyenaDNA} \citep{nguyen2023hyenadna} is a genome foundation model, operating at single-nucleotide resolution with context lengths up to $10^6$ to efficiently capture long-range dependencies beyond the reach of standard transformers. \textsc{DNABERT-2} \citep{zhou2023dnabert} is also a foundation model that replaces fixed $k$-mer tokenization with Byte Pair Encoding (BPE) to improve modeling efficiency. \textsc{DNABERT-S} \citep{zhou2024dnabert} leverages \textsc{DNABERT-2} as a pretrained backbone, fine-tuned with contrastive objectives tailored to metagenomic binning. \textsc{RevistKmers} \citep{ccelikkanat2024revisiting} is a lightweight model that learns sequence embeddings via a two-layer neural network applied to $4$-mer profiles. It provides a strong deterministic baseline and can be viewed as the non-probabilistic counterpart of our approach. \textsc{RevistKmers} is thus also the main baseline in the experimental setup as the proposed model shares the lightweight characteristics of \textsc{RevistKmers}.
%

\textbf{Architecture and training procedure.} The training of \textsc{\modelname} is conducted within a self-supervised contrastive learning framework by optimizing the objective function provided in Eq. \ref{eq:contrastive_loss} with probabilities given as in Eq. \ref{eq:algebraic_simplification}. We optimized the model on a cluster equipped with {NVIDIA A40 GPUs} and $120$GB of system memory.

The model utilizes two distinct encoder networks, $\phi_{\mu}$ and $\phi_{\sigma}$, which parameterize the mean vectors and diagonal covariance matrices of the multivariate Gaussian embeddings. Each encoder is implemented as a two-layer neural network designed to process input $k$-mer frequency profiles. Specifically, the mean and standard deviation branches both consist of a linear layer that projects the input $k$-mers into a $512$-unit hidden space, followed by batch normalization and a sigmoid activation function. To prevent overfitting, dropout ($p=0.2$) is applied before a final linear layer maps the features into the $256$-dimensional latent space ($D = 256$). To ensure a valid probabilistic representation, we apply an element-wise exponential transformation to the output of the final linear layer in the variance encoder to guarantee a strictly positive diagonal covariance matrix.

We optimize the network parameters using the Adam optimizer with a learning rate of $10^{-2}$. To enhance training stability, we adopt a sequential approach: the mean network is first trained alone for $50$ epochs, followed by the variance network for an additional $25$ epochs. Training is performed with a batch size of $10^5$ pairs. Each positive pair is derived by splitting a single DNA fragment into two halves to ensure they originate from the same genome. We generate $200$ negative pairs per positive sample via random sampling. During the initial mean-optimization stage, the success probability in Eq. \ref{eq:algebraic_simplification} is parameterized as $\exp(-\| \mathbf{z}_i - \mathbf{z}_j \|_2^2 / 4 )$ since we only optimize the mean encoder network. We additionally found it beneficial to adjust the scaling of the distance metric across optimization stages, and in the subsequent variance-optimization stage, the scaling factor in Eq. \ref{eq:algebraic_simplification} is omitted.

To stabilize training, we apply a margin-style threshold only to negative pairs ($y_{ij}=0$). When the distance term in the exponential exceeds the threshold $1$, we replace its contribution by a constant, yielding zero gradient for negative pairs that are already sufficiently separated. The exact constant is insignificant, since it is used only to stop further updates from uninformative negative pairs. This follows the standard margin intuition in contrastive metric learning: negative pairs need only satisfy a separation constraint, rather than be pushed indefinitely far apart.

It is important to note that this thresholding strategy is fully aligned with the theoretical guarantees presented in Section \ref{sec:model} of the main paper. In our distinguishability framework (Definition \ref{def:distinguishable}), each pairwise relation must satisfy the condition $P(Y_{ij} = y_{ij} \mid \mathbf{s}_i, s_j) \geq 1-\epsilon$. Once a negative pair becomes sufficiently separated, omitting its gradients simply indicates that the corresponding distinguishability requirement has already been fulfilled. Consequently, the optimization process concentrates on pairs 
that have not yet reached the desired separation regime and is therefore compatible with the pairwise distinguishability criterion of Definition~\ref{def:distinguishable}.

\begin{figure}[b]
     \centering
     \begin{subfigure}[b]{0.32\textwidth}
         \centering
         \includegraphics[width=\textwidth]{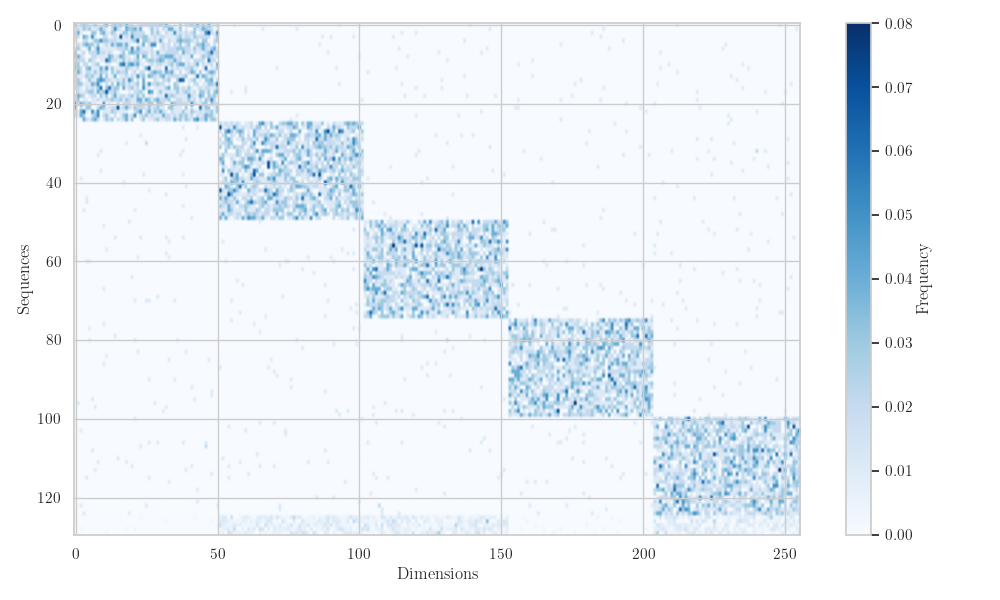}
         \caption{$k$-mer Frequencies}
         \label{fig:appendix_toy_kmer_freqs}
     \end{subfigure}
     \hfill
     \begin{subfigure}[b]{0.32\textwidth}
         \centering
         \includegraphics[width=\textwidth]{figures/toy_example_embedding_space.png}
         \caption{Learned Embeddings}
         \label{fig:appendix_toy_example_embs}
     \end{subfigure}
     \hfill
     \begin{subfigure}[b]{0.32\textwidth}
         \centering
         \includegraphics[width=\textwidth]{figures/toy_example_variance_distribution.png}
         \caption{Variance Distribution}
         \label{fig:appendix_toy_example_var_distr}
     \end{subfigure}
     \caption{Visualization of the input $k$-mer features with learned embeddings and the variances.}
     \label{fig:appendix_toy_example}
\end{figure}

\subsection{Toy Example with $k$-Mer Dataset}\label{appendix:toy_example}

To evaluate the behavior of our model in a controlled setting, we designed a synthetic toy dataset of $k$-mer sequences. The goal of this study is twofold: (i) to verify that the model can learn meaningful low-dimensional embeddings that reflect the underlying class structure, and (ii) to assess how the model represents sequences belonging to multiple classes with their mean and variance terms.

The dataset consists of sequences generated from multinomial distributions with a clear class structure, enabling us to systematically analyze the effect of sequence overlap and class separability on the learned embeddings.

\textbf{Data Generation.} We generated sequences of length $100$ composed of $4$-mers, resulting in $256$ possible $k$-mer types. The sequences were then divided into $5$ classes, each containing $25$ sequences. For each class, a multinomial distribution was defined such that the probability mass was concentrated on a distinct subset of k-mer dimensions, with a small uniform smoothing $10^{-2}$ to avoid zero probabilities. It ensured that sequences within the same class had similar $k$-mer compositions, while sequences from different classes were distinguishable (Figure \ref{fig:appendix_toy_kmer_freqs}). To further test the model's ability to handle ambiguity, we introduced $5$ “multi-class” sequences sampled by combining $k$-mer counts from multiple classes. This design simulates sequences that do not belong exclusively to a single class and allows us to probe how the model handles overlapping class structures.

\textbf{Pair construction.} We constructed positive sequence pairs by sampling two sequences from the same class, effectively assuming full access to positive examples (i.e., $y_{ij}=1$). Negative pairs were generated using the same random sampling strategy described in Section \ref{sec:experiments}, which naturally introduces the possibility of false negatives. For the multi-class sequences, we additionally constructed pairs with sequences from their contributing classes, enabling the model to learn representations that account for both pure and mixed class memberships.

\textbf{Results.} 
We learn sequence embeddings directly in a $2$-dimensional latent space (Figure~\ref{fig:appendix_toy_example_embs}), avoiding any dimensionality reduction steps that could distort the geometric structure of the embedding space. Visualization of the learned embeddings reveals that sequences from distinct classes form well-separated groups, while sequences belonging to multiple classes occupy intermediate regions. Importantly, our model also predicts a diagonal covariance matrix for each sequence, capturing the uncertainty associated with sequences that span multiple clusters.

Lemma~\ref{lemma:bounds} shows how the variance terms control pairwise distinguishability, while Theorem~\ref{thm:theorem} characterizes their behavior for a sequence that must remain compatible with multiple mutually negative sequences. Consistent with this qualitative behavior, the multi-class sequences exhibit substantially larger variances in Figure~\ref{fig:appendix_toy_example_var_distr}. This suggests that the model uses the covariance terms to represent ambiguity that cannot be captured by the mean embeddings alone.

\textbf{Comparison of distance metrics for ambiguous fragments.} Empirically, ambiguous fragments (labeled as "black points" in Figure \ref{fig:appendix_toy_example_embs}) lie in intermediate latent regions between clusters and are assigned significantly larger covariance (Figure \ref{fig:appendix_toy_example_var_distr}). While a single Gaussian distribution cannot explicitly represent multiple modes, this increased covariance adjusts the perceived distance in the embedding space. Under the standard Euclidean ($\ell_2$) distance, these ambiguous points appear numerically far from the primary clusters. However, the generalized Mahalanobis distance used in \textsc{\modelname} reflects their inherent uncertainty, providing a more faithful representation.  Table \ref{tab:relative_distance_comparison} presents the average pairwise distance between the set of ambiguous (black) fragments and each primary genomic class, normalized by the corresponding intra-class distance.

The normalization highlights the relative proximity: under Euclidean distance, ambiguous points are roughly $9.0$ to $12.7$ times further away from their neighboring classes than those classes' internal members are from each other. Under the Mahalanobis metric, this relative distance drops to between $1.5$ and $2.0$, indicating that when uncertainty is accounted for, the ambiguous fragments are recognized as being much closer to their constituent genomic distributions.

\begin{table}[t]
    \centering
    \caption{\textbf{Normalized average distances from ambiguous fragments to other clusters.} Values are computed by dividing the average inter-class distance by the intra-class distance of the target cluster.}\label{tab:relative_distance_comparison}
    \begin{tabular}{rccccc}
    \toprule
        Classes & Blue & Orange & Green & Red & Purple \\ \midrule
        $\ell_2$-distance & $80.3$ & $10.1$ & $9.0$ & $87.9$ & $12.7$ \\
        Mahalanobis & $10.7$ & $1.5$ & $2.0$ & $11.5$ & $1.7$ \\\bottomrule
    \end{tabular}
\end{table}
\begin{figure}[b]
\centering
\includegraphics[width=0.92\textwidth]{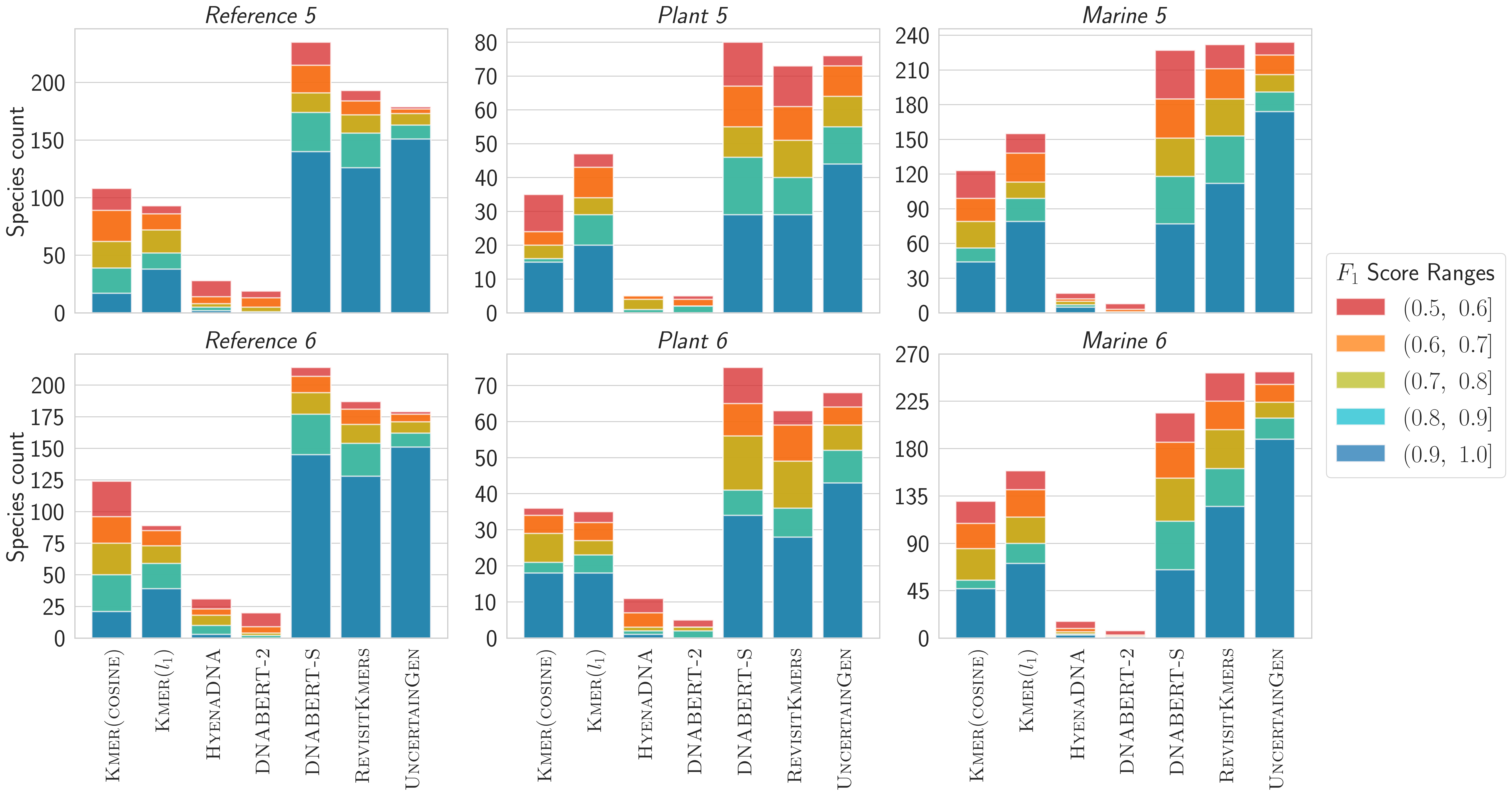}
\caption{\textbf{Metagenomic binning results.} Cluster counts are segmented by $F_1$-score quality ranges. The dark blue portion highlights the highest-quality bins for each model-dataset combination.}
\label{fig:all_metagenomic_binning_results}
\end{figure}

\subsection{Metagenomics Binning}\label{appendx:sec:metagenomics_binning}
We evaluate our methods on the metagenomic binning task, which aims to cluster genomic sequences into species-level groups without prior knowledge of the number of clusters. To this end, we adopt the modified K-Medoid algorithm proposed by \cite{zhou2024dnabert}, which jointly infers cluster assignments and the underlying number of species. This setting is particularly challenging, as it requires learned representations to be highly discriminative while remaining resilient to noise and inaccuracies in the estimated clustering structure.

To compute pairwise similarities between sequences, we use cosine similarity for all genome-scale foundation models as well as for the \textsc{KMers(cosine)} baseline. For \textsc{KMers($\ell_1$)}, we employ an exponential kernel applied to the $\ell_1$ distance, while \textsc{RevisitKMers} uses the $\ell_2$ distance (i.e., $\exp( -d_{\mathsf{X}}(\mathbf{z}_i, \mathbf{z}_j) )$ where $\mathbf{z}_i$ and $\mathbf{z}_j$ are the $k$-mer features or sequence embeddings depending on the model, and $d_{\mathsf{X}}(\cdot, \cdot)$ is the corresponding metric). For our proposed method, we adopt an exponential kernel over the generalized Mahalanobis term defined in Eq. \ref{eq:algebraic_simplification}, which provides a natural similarity measure in the learned latent space.

Following established evaluation protocols, we stratify predicted clusters into five quality tiers based on their $F_1$ scores. High-quality bins, defined as clusters with $F_1 > 0.9$, are highlighted in dark blue in Figure~\ref{fig:all_metagenomic_binning_results}. Across datasets, \textsc{\modelname} consistently outperforms its deterministic counterparts, especially the large genome foundation models.

\textbf{Statistical Stability.} To assess the reproducibility and reliability of our findings, we report the mean species counts and corresponding standard deviations across five independent experimental runs. The results, stratified by the $F_1$ quality ranges defined in Section 4.2, are summarized in Table \ref{tab:stability_results}. These results indicate that the performance gains achieved by \textsc{\modelname} are robust to initialization.

\begin{table}[h]
    \centering
    \caption{Mean and standard deviation (Std) of detected species counts by \textsc{\modelname}.}\label{tab:stability_results}
    \begin{tabular}{rcccccc}
    \toprule
        ~ & ~ & $F_1 > 0.5$ & $F_1> 0.6$ & $F_1 > 0.7$ & $F_1 > 0.8$ & $F_1 > 0.9$ \\ \midrule
        \multirow{2}{*}{\textsl{Reference 5}} & $F_1$ & $179$ & $176.8$ & $172.6$ & $163.4$ & $150.6$ \\ 
        ~ & Std & $2.191$ & $1.72$ & $2.332$ & $1.625$ & $2.332$ \\ \cmidrule(rl){2-7}
        \multirow{2}{*}{\textsl{Reference 6}} & $F_1$ & $178.6$ & $177$ & $171$ & $162.4$ & $151.2$ \\ 
        ~ & Std & $1.356$ & $1.414$ & $1.265$ & $1.625$ & $2.713$ \\ \cmidrule(rl){2-7}
        \multirow{2}{*}{\textsl{Plant 5}} & $F_1$ & $76.2$ & $73.4$ & $64.2$ & $54.8$ & $43.8$ \\ 
        ~ & Std & $2.857$ & $2.059$ & $2.638$ & $1.47$ & $2.135$ \\ \cmidrule(rl){2-7}
        \multirow{2}{*}{\textsl{Plant 6}} & $F_1$ & $67.8$ & $64$ & $58.8$ & $51.6$ & $43.4$ \\ 
        ~ & Std & $1.47$ & $1.673$ & $1.72$ & $0.49$ & $0.8$ \\ \cmidrule(rl){2-7}
        \multirow{2}{*}{\textsl{Marine 5}} & $F_1$ & $233.6$ & $223$ & $206.4$ & $191.2$ & $174.4$ \\ 
        ~ & Std & $3.72$ & $2.449$ & $2.245$ & $2.926$ & $1.497$ \\ \cmidrule(rl){2-7}
        \multirow{2}{*}{\textsl{Marine 6}} & $F_1$ & $253.2$ & $240.8$ & $224.4$ & $209.2$ & $188.8$ \\ 
        ~ & Std & $3.487$ & $4.445$ & $4.176$ & $3.124$ & $3.868$ \\\bottomrule 
    \end{tabular}
\end{table}

\subsection{Ablation Studies}\label{appendix:ablation_studies}

\textbf{Distribution of variances.} We begin our uncertainty analysis by inspecting the distribution of the predicted variance terms. Recall that for a given sequence $\mathbf{s}\in \mathcal{S}$, the model produces dimension-wise variance estimates $\phi_\sigma(\mathbf{s})$ corresponding to the diagonal entries of the covariance matrix. We compute
\begin{align}\label{eq:log_determinant}
    u(\mathbf{s}) = \sum_{d=1}^D \log \big(\phi_\sigma(\mathbf{s})_d + 1\big),
\end{align}
which is in fact the log-determinant of the covariance matrix, and we have the $+1$ term in order to ensure numerical stability. Therefore, it captures the sequence-level dispersion in the latent space.

Figure \ref{fig:appendix_log_determinant_distribution} shows the empirical distribution of $u(\mathbf{s})$ over sequences in the testing datasets. The multimodal distributions of $u(\mathbf{s})$ are very clear for \textsl{Reference 5/6} and \textsl{Plant 5/6} datasets.  This indicates that the model partitions the sequence space into distinct regimes of predictive uncertainty, which might point out the two distinct sets of species (Please see Section \ref{sec:experiments} for the dataset details.). Importantly, the observed distributions are neither degenerate (collapsed near zero) nor uniform. Instead, they encode structured variability that reflects properties of the underlying data distribution. This finding supports the hypothesis that the variance terms carry semantically meaningful information rather than merely acting as nuisance parameters. Hence, these results establish that our model produces non-trivial and interpretable uncertainty estimates.

\begin{figure}[t]
 \centering
 \includegraphics[width=0.9\textwidth]{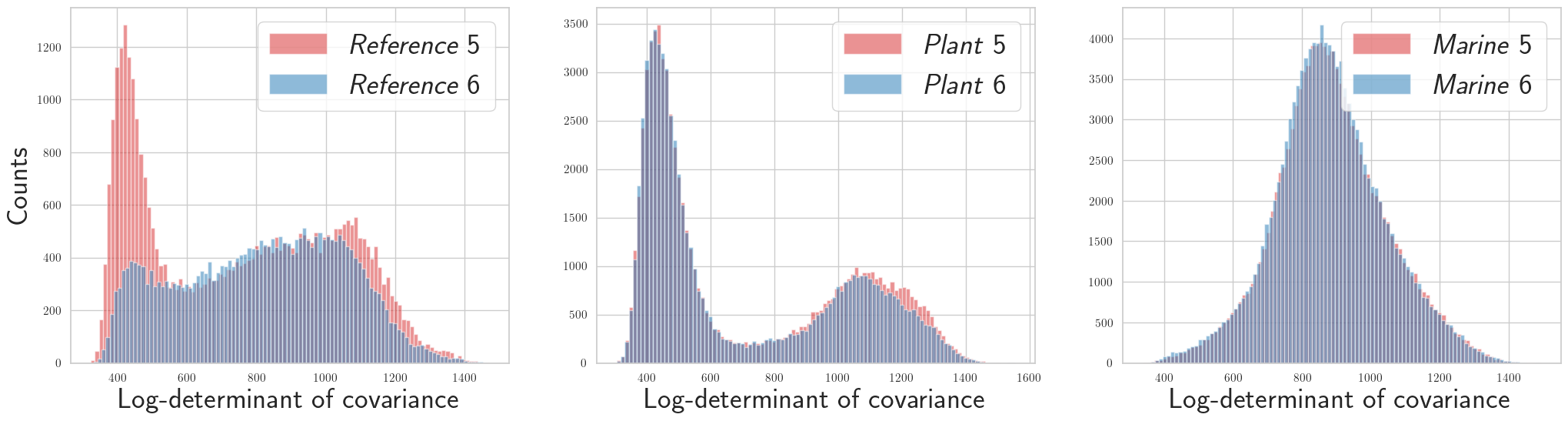}
\caption{Distribution of the log-determinant of covariance matrices across datasets.}
\label{fig:appendix_log_determinant_distribution}
\end{figure}

\begin{figure}[t]
 \centering
 \includegraphics[width=0.76\textwidth]{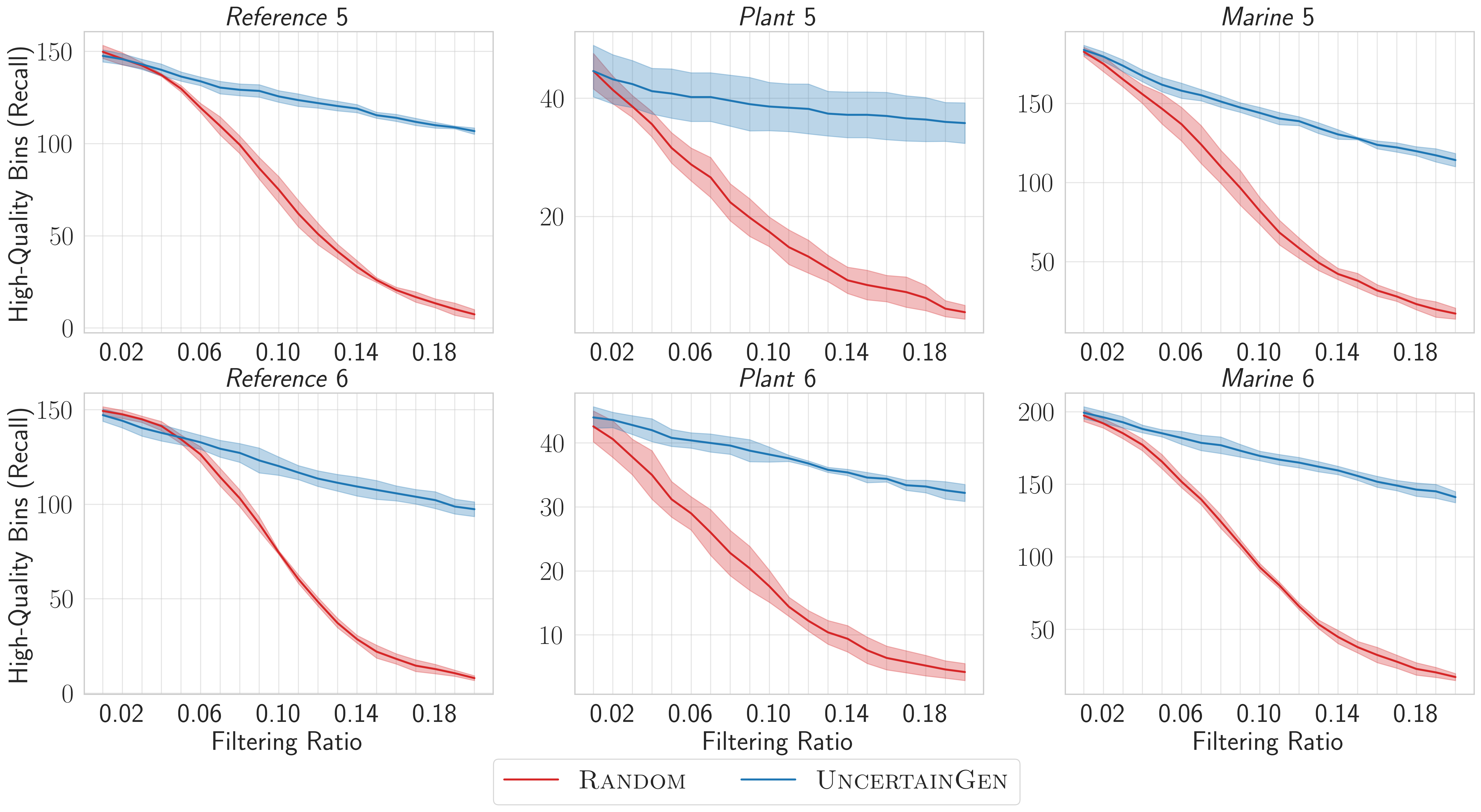}
\caption{Filtering sequences with varying ratios over all testing datasets.}
 \label{fig:appendix_filtering}
\end{figure}

\textbf{Sequence Filtering Across Datasets.} To further probe the role of predictive uncertainty, we conduct a sequence filtering experiment in which we selectively remove sequences with the largest variance scores. Specifically, we sort sequences by their aggregated log-determinant values, $u(\mathbf{s})$ (as defined in Eq. \ref{eq:log_determinant},) and iteratively filter out the highest-uncertainty items. After filtering, we evaluate cluster quality by reporting the number of clusters that achieve a recall score $\geq 0.9$. For comparison, we include a \emph{random filtering} baseline in which the same number of sequences is removed uniformly at random. To ensure a fair comparison, all removed sequences are assigned to a dedicated "garbage" cluster so that they do not artificially inflate false negatives.

Figure \ref{fig:appendix_filtering} summarizes results for all the benchmarks, and we observe the consistent trends across all datasets. Filtering by predictive uncertainty consistently yields a larger number of clusters with recall $\geq 0.9$ compared to the random baseline. This suggests that uncertainty-guided filtering preserves the integrity of high-quality clusters while selectively removing sequences that would otherwise degrade cluster purity. The sequences assigned the highest variance values might typically originate from low-quality or noisy bins. These sequences might also tend to coincide with cases that contribute to false negatives in the unfiltered setting. By removing them, the model effectively reduces noise in the evaluation and highlights clusters that better reflect true structure in the data. Therefore, the model’s variance estimates can serve as a practical mechanism for uncertainty-aware sequence selection and downstream decision making.

\textbf{Effect of embedding dimension size.} We analyzed the scalability of \textsc{\modelname} by evaluating the number of high-quality bins recovered across a range of latent dimension sizes, from $2^2$ to $2^9$, as illustrated in Figure~\ref{fig:appendix_dimension_impact}. Across all six benchmarks, we observe a consistent trend where performance improves sharply as the dimensionality increases from $4$ to $64$. This suggests that a minimum threshold of representational capacity is required to effectively model the input sequences. Most datasets reach a peak or a stable plateau between $64$ and $128$ dimensions. The \textsl{Marine} benchmarks exhibit a more robust plateau, maintaining high performance even at the largest evaluated dimension size.

\begin{figure}[h]
 \centering
 \includegraphics[width=0.76\textwidth]{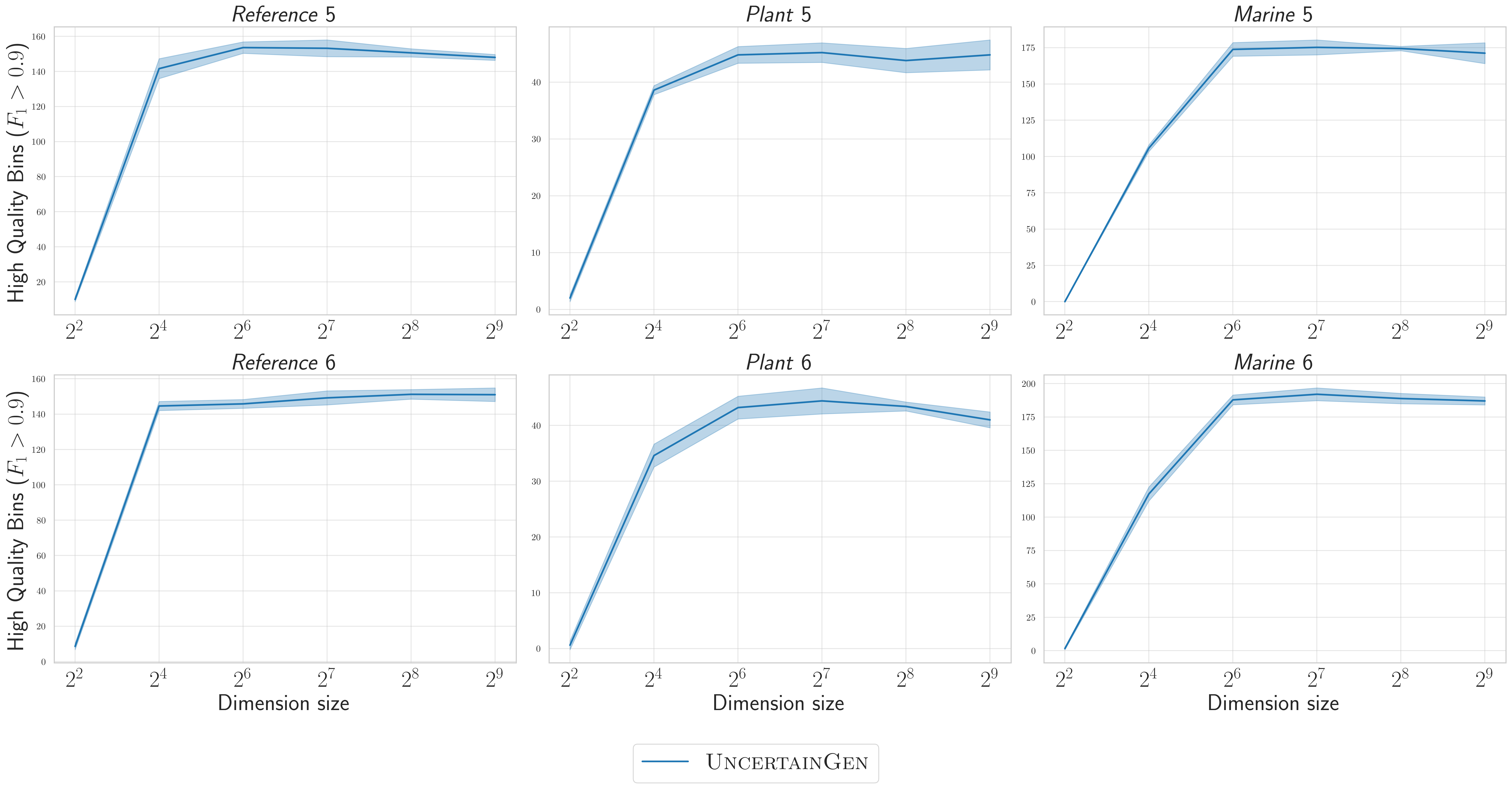}
\caption{Impact of dimension size for different metrics.}
 \label{fig:appendix_dimension_impact}
\end{figure}

\textbf{Computational Efficiency.} Our model is intentionally lightweight, consisting of two small neural networks. In contrast to large genome foundation models, this design leads to substantial improvements in runtime efficiency. To quantitatively assess computational efficiency, we compare the number of Multiply-Accumulate operations (MACs) required by our model and several baselines. For reference, $1$ MAC corresponds approximately to $2$ FLOPs. All measurements are computed with \texttt{ptflops}\footnote{https://github.com/LukasHedegaard/ptflops} Python package. To ensure a fair comparison, all models were evaluated on a randomly generated DNA sequence of size $1000$ bp. For foundation models, this was translated to the appropriate number of tokens based on their respective tokenization schemes (e.g., $1000$ tokens for \textsc{HyenaDNA}, ~$214$ for \textsc{DNABERT-2}). Since \textsc{RevisitKMers} and \textsc{\modelname} convert the given DNA sequence into a fixed $256$-dimensional $4$-mer frequency vector, the input size remains the same for these models regardless of the original input sequence length.

As shown in Table \ref{tab:mags}, large genome foundation models require roughly between $6$ to $25$ Giga-MACs per sequence, which is over four orders of magnitude more computation than our lightweight alternatives. In contrast, both \textsc{RevisitKMers} and our method, \textsc{\modelname}, achieve extremely low computational costs (placing in the roughly $10^{-4}$ to $10^{-3}$ Giga-MACs range). \textsc{\modelname} maintains this extreme efficiency while additionally modeling uncertainty in the embedding space. These results highlight that our approach provides a highly favorable trade-off between computational efficiency and representational power (Figure \ref{fig:performancevssize} of the main text), making it well-suited for large-scale genomic applications where both accuracy and scalability are critical.

\begin{table}[!ht]
    \centering
    \caption{Comparison of computational cost in terms of Multiply-Accumulate operations (MACs) across different models. Lower values indicate higher computational efficiency.}\label{tab:mags}
    \begin{tabular}{cccccc}
    \toprule
        \rotatebox{80}{Methods} & \rotatebox{80}{\textsc{HyenaDNA}} & \rotatebox{80}{\textsc{DNABERT-2}} & \rotatebox{80}{\textsc{DNABERT-S}} & \rotatebox{80}{\textsc{RevisitKMers}} & \rotatebox{80}{\textsc{\modelname}} \\ \cmidrule(rl){1-1}\cmidrule(rl){2-2}\cmidrule(rl){3-3}\cmidrule(rl){4-4}\cmidrule(rl){5-5}\cmidrule(rl){6-6}
        MACs (G) & $6.55$ & $25.92$ & $25.12$ & $0.26 \times 10^{-3}$ & $0.53 \times 10^{-3}$ \\ 
        \bottomrule
    \end{tabular}
\end{table}

\section{Impact Statement}\label{sec:impact_statement}
This paper presents work whose primary goal is to advance machine learning methods for representation learning on genomic sequence data. We introduce an uncertainty-aware embedding framework for DNA sequences and use metagenomic binning solely as a benchmark task to evaluate the quality of the learned representations. We do not propose or deploy a full metagenomic binning pipeline, nor do we claim direct applicability to clinical or decision-critical settings.

The anticipated societal and ethical implications of this work are consistent with those commonly associated with methodological advances in machine learning and computational genomics, and we do not foresee any immediate negative impacts that require specific discussion.





\end{document}